%% file: main.tex
\documentclass{article}

\usepackage[utf8]{inputenc}
\usepackage[stable]{footmisc}
\usepackage{amsmath}  
\usepackage{graphicx} 
\usepackage{amssymb,amsthm}
\usepackage{bm,bbm}
\usepackage{booktabs}
\usepackage{dsfont}
\usepackage{graphicx}
\usepackage{subcaption}
\usepackage{blindtext}
\usepackage{url}
\usepackage{algorithm,algorithmic}
\usepackage{setspace}
\usepackage{enumitem}
\usepackage{thmtools}
\usepackage{thm-restate}
\usepackage{authblk}
\usepackage{comment}
\usepackage{geometry}
\usepackage{array}

\usepackage[normalem]{ulem}

\input{Macros}

\usepackage{cite} 
\usepackage{hyperref}
\hypersetup{
	final,
    colorlinks=true,
    linkcolor=blue,
    filecolor=magenta,
    urlcolor=cyan,
}

\usepackage{blindtext}
\usepackage{url}
\usepackage{xcolor}

\title{{\usefont{OT1}{bch}{b}{n}
	\LARGE Defending Against Diverse Attacks in Federated Learning Through Consensus-Based Bi-Level Optimization}}

\date{}

\author[1]{Nicol\'as Garc\'ia Trillos\thanks{Email: \texttt{garciatrillo@wisc.edu}}}
\author[2]{Aditya Kumar Akash\thanks{Email: \texttt{adityakumarakash@gmail.com}}}
\author[1]{Sixu Li\thanks{Email: \texttt{sli739@wisc.edu}}}
\author[3]{Konstantin Riedl\thanks{Email: \texttt{Konstantin.Riedl@maths.ox.ac.uk}}}
\author[4]{Yuhua Zhu\thanks{Email: \texttt{yuhuazhu@ucla.edu}}}

\affil[1]{University of Wisconsin-Madison, Department of Statistics}
\affil[2]{Google}
\affil[3]{University of Oxford, Mathematical Institute}
\affil[4]{University of California, Los Angeles, Department of Statistics and Data Science}

\begin{document}

\maketitle

\begin{abstract}
    \noindent
    Adversarial attacks pose significant challenges in many machine learning applications, particularly in the setting of distributed training and federated learning, where malicious agents seek to corrupt the training process with the goal of jeopardizing and compromising the performance and reliability of the final models.
In this paper, we address the problem of robust federated learning in the presence of such attacks by formulating the training task as a bi-level optimization problem.
We conduct a theoretical analysis of the resilience of consensus-based bi-level optimization (CB\textsuperscript{2}O), an interacting multi-particle metaheuristic optimization method, in adversarial settings.
Specifically, we provide a global convergence analysis of CB\textsuperscript{2}O in mean-field law in the presence of malicious agents, demonstrating the robustness of CB\textsuperscript{2}O against a diverse range of attacks.
Thereby, we offer insights into how specific hyperparameter choices enable to mitigate adversarial effects.
On the practical side, we extend CB\textsuperscript{2}O to the clustered federated learning setting by proposing FedCB\textsuperscript{2}O, a novel interacting multi-particle system,
and design a practical algorithm that addresses the demands of real-world applications.
Extensive experiments demonstrate the robustness of the FedCB\textsuperscript{2}O algorithm against label-flipping attacks in decentralized clustered federated learning scenarios, showcasing its effectiveness in practical contexts.
\end{abstract}



\tableofcontents

\section{Introduction}\label{sec:intro}
Adversarial attacks, such as data poisoning~\cite{steinhardt2017certified,fung2020limitations,bagdasaryan2020backdoor,wang2020attack}, backdoor attacks~\cite{chen2017targeted,bagdasaryan2020backdoor,xie2019dba,wang2020attack}, evasion attacks~\cite{biggio2013evasion,cao2017mitigating}, membership inference attacks, or several others~\cite{nasr2019comprehensive,geiping2020inverting,shumailov2021manipulating}, pose serious threats to the performance, reliability, and integrity of many machine learning (ML) models.
This raises severe safety concerns due to the widespread use of technology enhanced by artificial intelligence in applications such as personal health monitoring~\cite{lian2022deep}, autonomous driving~\cite{nguyen2022deep,dai2023online}, large language models~\cite{kasneci2023chatgpt,thirunavukarasu2023large}, and more.
For instance, agents with malicious intentions may try to contaminate training datasets with samples that are meticulously designed to enforce specific errors in a model's outputs, or try to alter test samples by unrecognizable perturbations, so-called adversarial examples~\cite{szegedy2013intriguing, goodfellow2014explaining}, to fool trained models during inference. In distributed training and federated learning (FL), in particular, the decentralized nature of the training process increases the vulnerability of models to a diverse range of adversarial attacks~\cite{tolpegin2020data,nasr2019comprehensive,bagdasaryan2020backdoor, geiping2020inverting}.

%
FL~\cite{mcmahan2017communication,konevcny2016federated,kairouz2021advances, beltran2023decentralized} is a distributed ML paradigm that enables model training directly on those devices where the data was originally generated.
It has been developed to overcome the inefficiencies of centralized data collection and model training while preserving data privacy and security of the participants.
One popular FL paradigm is  decentralized federated learning (DFL) \cite{he2018cola,kovalev2021linearly, barbieri2022decentralized,beltran2023decentralized, beltran2024fedstellar}.
Unlike centralized approaches \cite{mcmahan2017communication,kairouz2021advances}, DFL operates without a central server, relying only on direct interactions between agents that adhere to the following two main steps in each communication round: (i) {Local update step:} Agents update their models locally on their own device using their private stored datasets, typically through running a few epochs of stochastic gradient descent or another ML optimization algorithm; (ii) {Model exchange and local aggregation step:} Agents then share their locally updated models with others
and aggregate the models they receive to improve their own local ones. A pictogram of the DFL framework is provided in Figure~\ref{subfig:DFL}.
\setlength{\abovecaptionskip}{5pt}
\setlength{\belowcaptionskip}{0pt}
\begin{figure}[!htb]
	\centering
	\subcaptionbox{\label{subfig:DFL}DFL}{
		\includegraphics[trim=300 56 300 56,clip, width=0.46\textwidth]{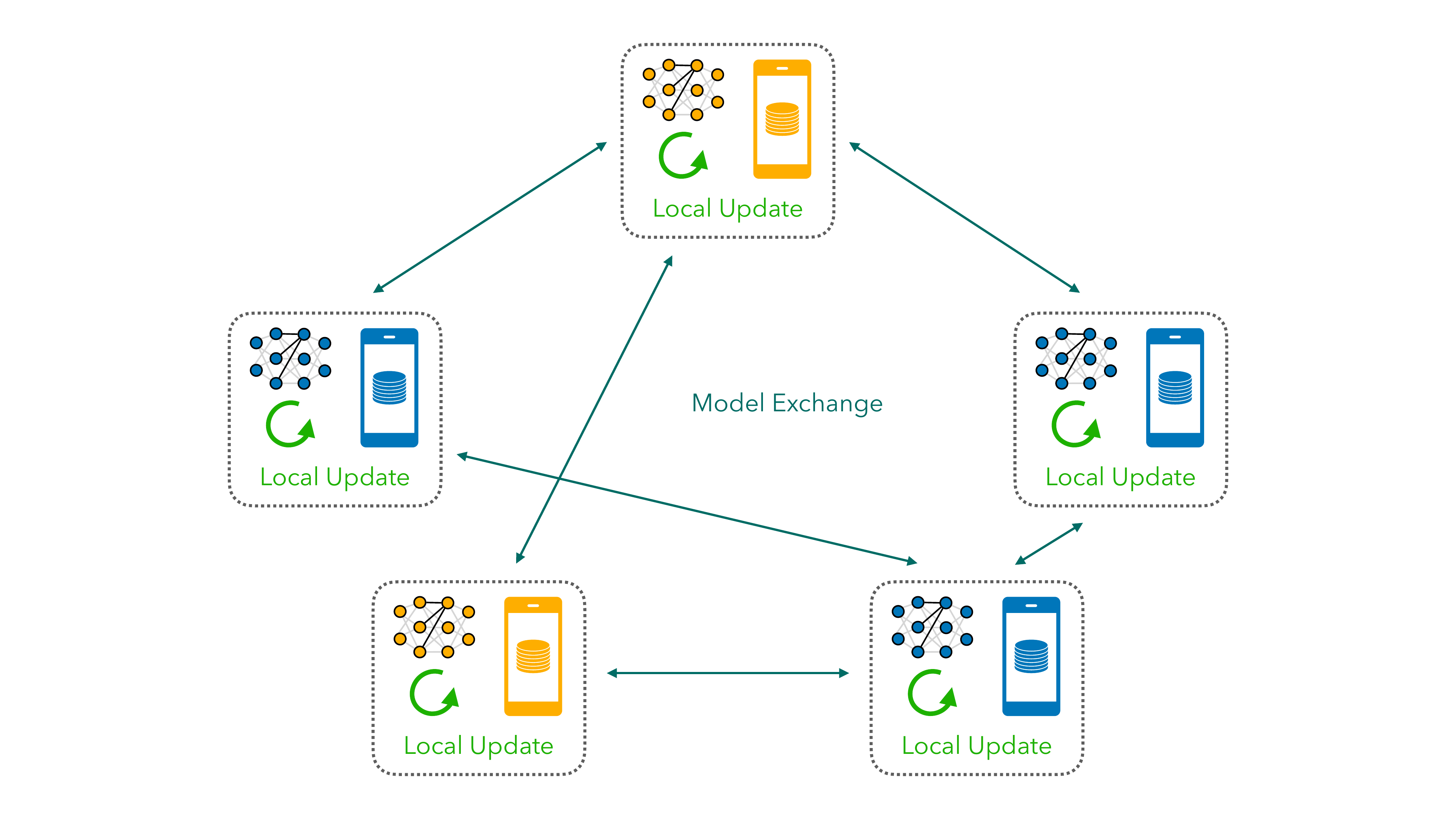}
	}
    \hspace{2em}
	\subcaptionbox{\label{subfig:DFL_attack}DFL in the presence of malicious agents}{
		\includegraphics[trim=300 56 300 56,clip, width=0.46\textwidth]{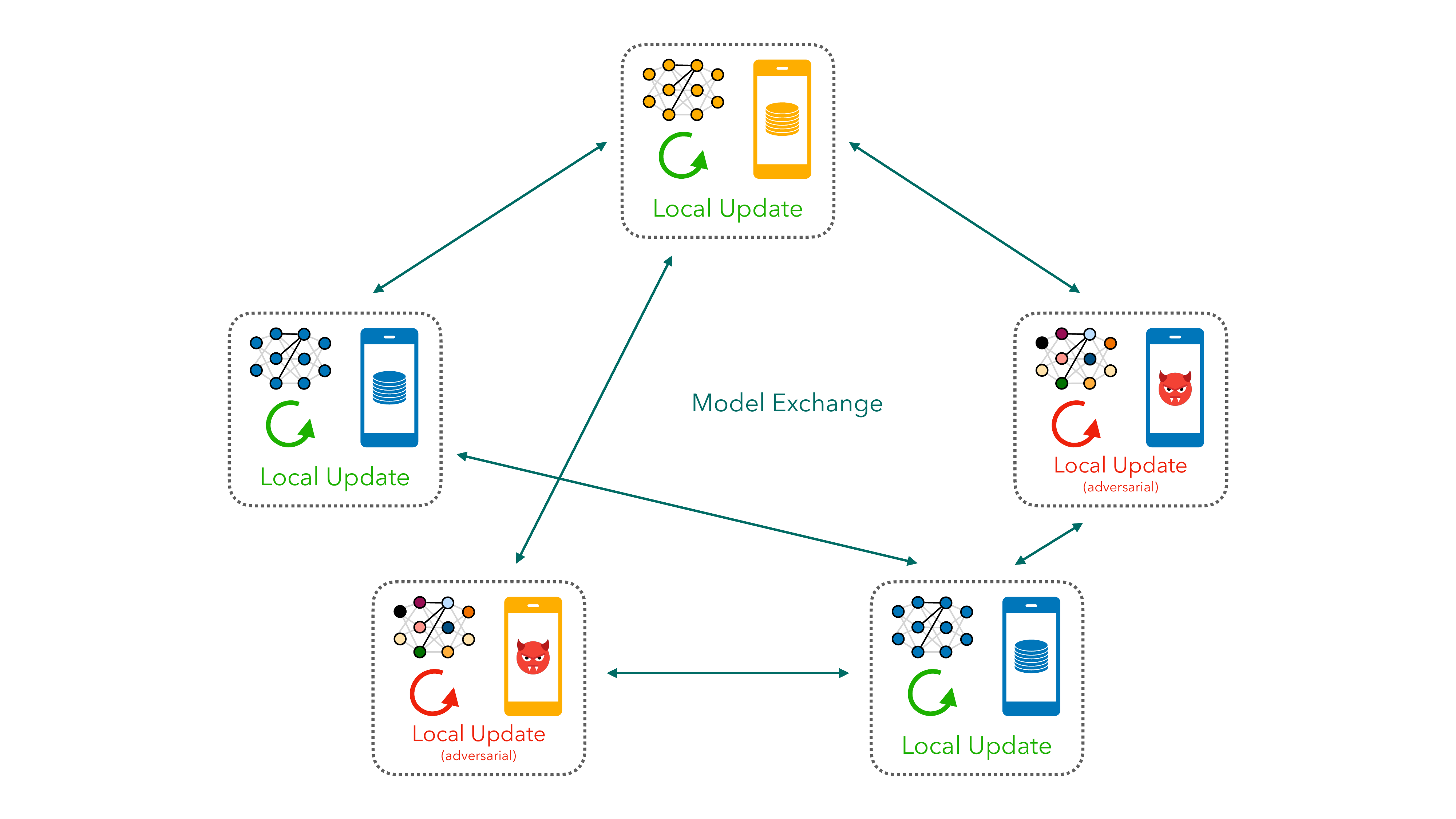}}
	\caption{A pictogram of the decentralized federated learning (DFL) paradigm.}
\end{figure}

The decentralized nature of DFL enhances the communication efficiency while better preserving data privacy of the individual agents through keeping their individually stored data localized.
However, this makes the system more vulnerable to poisoning attacks from malicious agents \cite{Lyu2020, taxonomy2021jere,tolpegin2020data}; see Figure~\ref{subfig:DFL_attack} for an illustration.
These malicious agents can inject arbitrary but specifically designed ``poisoned'' models into the system.
For the benign agents, verifying the authenticity of these models becomes challenging because the local training data and training processes of other participants are hidden in the system.
Without carefully filtering out such poisoned models,  incorporating them during the aggregation steps will degrade the performance of the final trained models of benign agents.

From the perspective of a benign agent, one straightforward strategy to assess the usefulness and trustworthiness of a model is to first evaluate its performance on the agent's own local dataset, which we represent through a loss function~$L$, and consecutively utilize only those models with relatively small average loss~$L$; see \cite{carrillo2024fedcbo}. To simplify our exposition, in what follows we assume that all agents have the same loss function $L$ and refer to Remark~\ref{rem:DCFL} for a discussion on the more realistic setting of different agents possessing different objective functions~$L$.
The approach of relying on $L$ to filter simple poisoning attacks (for instance, when malicious agents send ``trash'' models, such as randomly generated models) may be effective since in those cases the received models will generally perform poorly on the datasets of benign agents. 
However, this strategy becomes insufficient against more sophisticated attacks, such as label-flipping attacks \cite{fung2020limitations, tolpegin2020data, jebreel2023fl, JEBREEL2024111}, where poisoned models may still achieve low average loss~$L$ on the datasets of benign agents while embedding harmful biases for specific classes; see Section~\ref{sec:algorithm} for more detailed explanations.
These challenges underscore the need for a more flexible and robust framework that benign agents in FL systems can leverage to defend against advanced attacks. 

Motivated by the above discussion, in this paper, we propose to incorporate a secondary layer of evaluation for the benign agents to assess the trustworthiness of models from other agents. This layer is implemented, for a given type of attack, through a suitable robustness criterion encoded by an upper-level objective function~$G$. We thereby abstract the robust training task in the DFL setting and formulate it as a bi-level optimization problem of the form 
\begin{equation}
    \label{eq:bilevel_opt}
    \thetaG := \argmin_{\theta^*\in \Theta} G(\theta^*)
    \quad \text{s.t.\@}\quad
    \theta^* \in \Theta := \argmin_{\theta \in \bbR^d} L(\theta)
\end{equation}
under the FL paradigm.
If malicious agents now perform more elaborate attacks and share poisoned models achieving small average loss on the datasets of benign agents, i.e., models that yield a good value for the lower-level objective function~$L$, the robustness criterion $G$ will serve as a tool to separate these poisoned models from benign ones.
Enhancing robustness can thus be mathematically formulated as finding the global minimizer $\thetaG$ of $G$ within the set $\Theta$ of global minimizers of $L$ (in practice, approximate minimizers).
In this way, problem~\eqref{eq:bilevel_opt} serves as a mathematical device to tilt the benign agents' preferences towards models that satisfy the additional robustness properties implemented through $G$. Naturally, different choices of $G$ may be required to defend against different types of attacks. 

To solve a bi-level optimization problem of the form \eqref{eq:bilevel_opt} in the FL context and to analyze the impact of malicious agents on the system, 
we combine the viewpoints from \cite{carrillo2024fedcbo}, which studies FL from the perspective of interacting particle systems (IPS), with those of consensus-based bi-level optimization (CB\textsuperscript{2}O)~\cite{trillos2024CB2O}, an IPS-based approach~\cite{pinnau2017consensus,carrillo2018analytical,fornasier2021consensus,huang2022global,riedl2023all,riedl2023all2} specifically designed to solve optimization problems of the form \eqref{eq:bilevel_opt} for potentially nonconvex upper- and lower-level objective functions.

Before providing a brief discussion on the CB\textsuperscript{2}O framework~\cite{trillos2024CB2O} and its capability to ensure robustness in FL systems, let us first comment on the generalization of the above setting to the more realistic one where different agents may have different objective functions~$L$.
\begin{remark}[Decentralized clustered federated learning (DCFL)]
    \label{rem:DCFL}
    In real-world scenarios, edge devices (users/agents) typically possess heterogeneous datasets, implying that different users have different lower-level objective functions $L$.
    Yet, it is reasonable to assume a certain (unknown) group structure among users (illustrated by the two different colors, yellow and blue, in Figure~\ref{subfig:DFL}), which reflects the idea that users with similar backgrounds are likely to make similar decisions and thus generate data following similar distributions (which translates to agents in the same group having similar loss functions). This idea is made precise by the clustered federated learning setting~\cite{sattler2020clustered,ghosh2020efficient,ruan2022fedsoft,long2023multi,ma2022convergence,carrillo2024fedcbo}. In this context, the goal is to propose communication protocols that can produce learning models for each cluster of users, rather than a single model for all users. Naturally, preserving data privacy is still an important constraint during the training process (in particular, group membership of agents is never revealed during training); see Section \ref{sec:algorithm} for a more detailed description of the clustered federated learning setting. The IPS-based approach designed in \cite{carrillo2024fedcbo} accommodates this clustered FL setting.
    As we discuss in Remark~\ref{rem:DCFL_FedCB2O_theory}, the analysis conducted in Section~\ref{sec:main} for the case of a single group of users can be extended to the DCFL setting by combining the results of \cite{carrillo2024fedcbo} with the ones of Section~\ref{sec:main} and \cite{trillos2024CB2O}.
    For simplicity, the theoretical analysis in Section~\ref{sec:main} will thus be restricted to the case where all users share a single lower-level objective $L$ (i.e., the homogeneous data setting), which is of interest in its own right.
\end{remark}

\paragraph{Consensus-Based Bi-Level Optimization.}
For a system with $N$ agents,  CB\textsuperscript{2}O~\cite{trillos2024CB2O} can be used to implement the training of the agents' models $\theta^1, \dots, \theta^N \in \R^d$. Formally, we describe the agents' model parameters as a system of time-evolving processes, i.e., $\theta^i=(\theta^i_t)_{t\geq0}$ for $i=1,\dots,N$, which, for user-specified parameters $\alpha,\beta, \lambda, \sigma > 0$, satisfy the system of stochastic differential equations~(SDEs) 
\begin{equation}
\label{eq:dyn_micro2}
\begin{split}
    d\theta_t^i
    &=
    -\lambda \left(\theta_t^i - \mAlphaBeta{\rho_t^N} \right)dt + \sigma D\left(\theta_t^i - \mAlphaBeta{\rho_t^N}\right) dB_t^i,
    \qquad\theta_0^i\sim \rho_0,
\end{split}
\end{equation}
where $((B_t^i)_{t\geq0})_{i=1,\dots,N}$ are independent standard Brownian motions in $\bbR^d$, and $\rho_t^N$ denotes the empirical measure of all model parameters at time $t$. Here, for an arbitrary probability measure $\varrho\in\CP(\bbR^d)$ the consensus point~$\mAlphaBetanoarg$ is defined according to
\begin{equation}
    \label{eq:consensus_point}
    \mAlphaBeta{\varrho}
    := \int \theta \frac{\omegaa(\theta)}{\N{\omegaa}_{L^1(\Ibeta{\varrho})}} d\Ibeta{\varrho}(\theta),
    \quad \text{with}\quad
    \omegaa(\theta):= \exp \left(-\alpha G(\theta) \right)
\end{equation}
and with $\Ibeta{\varrho}:=\mathbbm{1}_{\Qbeta{\varrho}} \varrho$,
where for some (fixed at initialization) parameters $\delta_q>0$, sufficiently small, and $R > 0$, sufficiently large, the set $\Qbeta{\varrho}$ is defined as
\begin{align}
    \label{eq:Q_beta}
    \Qbeta{\varrho} 
    &:=
    \left\{\theta\in B_R(0) :  L(\theta)\leq \frac{2}{\beta}\int_{\beta/2}^{\beta} \qa{\varrho}\,da +\delta_q \right\},
\end{align}
with the $a$-quantile function~$\qa{\varrho}$ of $\varrho$ under $L$ defined by
\begin{align}
    \label{eq:q_beta}
    \qa{\varrho}
    &:=
    \arginf_{q\in\bbR} \big\{a \leq \varrho(L(\theta)\leq q) \big\}.
\end{align}

To provide some intuition on the system \eqref{eq:dyn_micro2} and what it enforces, we first observe that the drift term (first term in \eqref{eq:dyn_micro2}) drives each individual agent to align its model parameter with the consensus point $\mAlphaBeta{\rho_t^N}$, which is a weighted average of all models in the system. This weighted average favors models from agents that have a small value of $L$ and that, in addition, attain a small value for the upper-level objective $G$. To see this, observe that the quantity $\frac{2}{\beta} \int_{\beta/2}^{\beta} \qbeta{\varrho} da + \delta_q$, as defined in \eqref{eq:Q_beta}, can be viewed as a proxy for $\underbar{L}$ (denoting, from now on, the infimum of $L$ over all possible model parameters) based on the currently available information from the density $\varrho$, provided that $\beta$ and $\delta_q$ are sufficiently small. Consequently, the sub-level set $\Qbeta{\varrho}$ can be interpreted as an approximation of the neighborhood of the set $\Theta$ of global minimizers of $L$ that is inferred from the information available in $\varrho$. Building upon this intuition, the expression for the consensus point $\mAlphaBeta{\varrho}$ can thus be understood as taking a weighted average w.r.t.\@ the upper-level objective $G$ within the neighborhood of the set $\Theta$ of global minimizers of the lower-level objective $L$, and the system \eqref{eq:dyn_micro2} can be thought of as a system of particles that jointly target the global minimizer $\thetaG$ of the bi-level optimization problem \eqref{eq:bilevel_opt}.
The diffusion term (second term in \eqref{eq:dyn_micro2}) is used, as in many optimization schemes, to induce exploration of the loss landscape.
Throughout this paper, $D(\dummy)=\N{\dummy}_2\Id$.
For more detailed explanations and additional insights into the design and rationale behind the CB\textsuperscript{2}O system~\eqref{eq:dyn_micro2}\,--\,\eqref{eq:q_beta}, we refer readers to \cite{trillos2024CB2O}.
There, it is also explained that, in practice, the CB\textsuperscript{2}O system can be simplified \revisedOne{(by setting $R=\infty$, $\delta_q=0$, and by replacing $\frac{2}{\beta}\int_{\beta/2}^{\beta} \qa{\varrho}\,da$ with $\qbeta{\varrho}$ and in the definition of $\Qbeta{\varrho}$ in \eqref{eq:Q_beta}),} while here we have formulated it in a form that is tractable for the type of rigorous mathematical analysis that we discuss shortly.

The mean-field limit of the finite particle system \eqref{eq:dyn_micro2}, i.e., as $N \rightarrow \infty$, is described by the stochastic process $\overbar{\theta} = (\overbar{\theta}_t)_{t\geq 0}$ satisfying the self-consistent nonlinear nonlocal SDE
\begin{equation}
\label{eq:dyn_macro_no_mali}
    d\overbar{\theta}_t
    =
    -\lambda \left(\overbar{\theta}_t - \mAlphaBeta{\rho_t} \right) dt + \sigma D\left(\overbar{\theta}_t - \mAlphaBeta{\rho_t}\right) dB_t,
    \qquad \overbar{\theta}_0 \sim \rho_0,
\end{equation}
where $\rho_t := \mathrm{Law} (\overbar{\theta}_t)$.
Under mild assumptions on the initial distribution~$\rho_0$ and minimal assumptions on the objective functions $L$ and $G$, \cite[Theorem~2.7]{trillos2024CB2O} proves that CB\textsuperscript{2}O converges in mean-field law to the target global minimizer $\thetaG$ of the bi-level optimization problem \eqref{eq:bilevel_opt}.

In the context of DFL, this result can be understood as follows.
If all agents, aiming to jointly train one model that minimizes the loss function $L$ while enjoying the robustness property encoded through $G$, \textit{strictly adhere} to the training protocol defined by the dynamics \eqref{eq:dyn_micro2},
they will, provided that the number of agents is sufficiently large, eventually converge to the target minimizer $\thetaG$ of the bi-level optimization problem \eqref{eq:bilevel_opt}.

\paragraph{Consensus-Based Bi-Level Optimization in the Presence of Attacks.}

However, when malicious agents, who \textit{purposely deviate} from the dynamics specified by \eqref{eq:dyn_micro2} (for instance, by executing adversarial or poisoning attacks), are present in the system, it is unclear whether the models of benign agents will still converge to the target minimizer $\thetaG$.
To make this question mathematically more precise, consider the (mean-field) system
\begin{subequations}\label{eq:dyn_macro}
\begin{align}
    d\bar{\theta}^b_t
    &=
    -\lambda \left(\bar{\theta}^b_t - \mAlphaBeta{\rho_t} \right) dt + \sigma D\left(\bar{\theta}^b_t - \mAlphaBeta{\rho_t}\right) dB^b_t, \label{eq:dyn_macro_b}\\
    d\bar{\theta}^m_t
    &= a_t\,dt + A_t\,dB^m_t,
    \label{eq:dyn_macro_m}
\end{align}
\end{subequations}
where $\rho_t = w_b\rho^b_t + w_m\rho^m_t$ with $\rho^b_t = \mathrm{Law} (\bar{\theta}^b_t)$ and $\rho^m_t = \mathrm{Law} (\bar{\theta}^m_t)$ satisfying
\begin{subequations}\label{eq:fokker_planck}
\begin{align}
    \partial_t\rho^b_t
	&= \lambda\divergence \left(\left(\theta -\mAlphaBeta{\rho_t}\right)\rho^b_t\right)
	+ \frac{\sigma^2}{2}\sum_{k=1}^d \partial_{kk} \left(D\left(\theta-\mAlphaBeta{\rho_t}\right)_{kk}^2\rho^b_t\right),\label{eq:fokker_planck_b}\\
    \rho^m
    &\in\CC([0,T],\CP_4(\bbR^d)).\label{eq:fokker_planck_m}
\end{align}
\end{subequations}
Here, $\overbar{\theta}^b$ represents a typical benign agent in the system who follows the robust training protocol defined by the CB\textsuperscript{2}O dynamics,
while $\overbar{\theta}^m$ denotes a generic malicious agent who deviates therefrom by executing attacks. We model the behavior of malicious particles $\overbar{\theta}^m$ through an SDE with an arbitrary drift $a_t$ and an arbitrary diffusion $A_t$ as in \eqref{eq:dyn_macro_m} to emphasize that malicious agents can behave arbitrarily and perform a wide range of attacks. The specific form of $a_t$ and $A_t$ is of no particular importance, as long as they are regular enough to ensure that the SDE is well-defined and that the corresponding law $\rho^m$ is continuous in time as in \eqref{eq:fokker_planck_m}, a requirement introduced solely to facilitate a rigorous theoretical analysis in Section~\ref{sec:main}.
The weights $w_b$ and $w_m$ represent the proportions of benign and malicious agents in the system. We note that the consensus point computed by benign agents takes into account all agents in the system, benign or malicious, given that benign agents have no a priori way to distinguish between them.

Based on the modified mean-field dynamics \eqref{eq:dyn_macro} and \eqref{eq:fokker_planck} with maliciously and irregularly behaving agents attacking the dynamics,
we pose the question: 
\begin{center}
    {\it``Can the benign particles (agents) $\overbar{\theta}^b$ still converge to their target minimizer $\thetaG$\\despite the presence of malicious particles (agents) $\overbar{\theta}^m$?''}
\end{center}

\subsection{Contributions} 
In this paper, we provide a positive answer to this question and support it by both theoretical analysis and experimental evidence.
On the theoretical side,
by conducting a global convergence analysis of the mean-field system \eqref{eq:dyn_macro} and \eqref{eq:fokker_planck} in the presence of malicious agents,
we demonstrate that consensus-based bi-level optimization (CB\textsuperscript{2}O) \cite{trillos2024CB2O} is robust against a wide range of attacks (Theorem~\ref{thm:main}).
Our analysis highlights in particular how benign agents can effectively mitigate adversarial effects by appropriately choosing the hyperparameters $\beta$ and $\alpha$ of the method. 
This establishes a rigorous theoretical foundation for the applicability of CB\textsuperscript{2}O in adversarial settings.
On the algorithmic side, to demonstrate the robustness of CB\textsuperscript{2}O practically, we tackle the problem of defending against poisoning attacks in the DCFL setting.
Building upon ideas from \cite{carrillo2024fedcbo}, where FedCBO is proposed for the attack-free DCFL problem, we extend the CB\textsuperscript{2}O dynamics to the clustered federated learning setting and propose a novel, robustified DCFL framework, which we call FedCB\textsuperscript{2}O.
In doing so, we develop a new agent selection mechanism to facilitate the consensus point computation while accounting for practical demands that arise in real-world FL applications.
This mechanism, which may be of independent interest to any DFL algorithm requiring agent selection, is integrated into the FedCB\textsuperscript{2}O algorithm (Algorithm~\ref{alg:FedCB2O}).
Through extensive experiments, we validate the effectiveness of the FedCB\textsuperscript{2}O algorithm in realistic FL scenarios in the presence of malicious agents performing label-flipping attacks.

\subsection{Related Works}
\paragraph{Attacks and Defenses in Federated Learning.}
Adversarial attacks in federated learning (FL) can be broadly classified into two main categories: attacks targeting federated models and privacy attacks \cite{rodriguez2023survey,hallaji2024decentralized}.
While the former class aims to degrade the performance and reliability of the final trained models, the latter seeks to reconstruct or infer the private data 
of other agents. 
This paper focuses on scenarios where malicious agents attempt to jeopardize the performance of an FL system during training time, commonly referred to as poisoning attacks \cite{rodriguez2023survey}.
Since both training data and the model training process of each agent are hidden from others, malicious agents can engage in two types of poisoning: {data poisoning attacks} and {model poisoning attacks}.
To execute a data poisoning attack, malicious agents manipulate their own local datasets used to train models.
They may do this by modifying the labels of a subset of the local training data (label-flipping attacks) \cite{tolpegin2020data, li2021lomar,li2021detection,jiang2023data,hallaji2023label,jebreel2023fl,JEBREEL2024111}, by altering data features by adding manually designed patterns to images \cite{chen2017targeted,bagdasaryan2020backdoor,sun2019can} or generating poisoned samples using generative models \cite{zhang2019poisoning,zhang2020poisongan}, which are collectively referred to as sample poisoning attacks.
To execute model poisoning attacks, malicious agents either randomly generate model parameters based on other agents' models \cite{fraboni2021free,fang2020local} or solve an optimization problem to maximize their attack's success while minimizing the differences between their poisoned models and other models in the system \cite{bhagoji2019analyzing,shejwalkar2021manipulating}.

To defend against poisoning attacks in FL, numerous strategies have been proposed.
The most common approach, which is relevant in our context, is to replace the simple averaging (mean) in the local aggregation step (as in FedAvg \cite{mcmahan2017communication}) with a robust aggregation operator to reduce sensitivity to outliers or extreme values.
This class of methods is commonly referred to as {robust aggregation} \cite{rodriguez2023survey, hallaji2024decentralized} and  includes techniques based on statistically robust estimators, such as median \cite{yin2018byzantine}, trimmed-mean \cite{yin2018byzantine}, geometric-mean \cite{wu2020federated,pillutla2022robust}, Krum and multi-Krum \cite{blanchard2017machine}.
It is worth noting that Sageflow \cite{park2021sageflow} employs entropy-based filtering and loss-weighted averaging during the local aggregation step, which is similar to the use of the weighted averaging w.r.t.\@ robustness criterion $G$ in our CB\textsuperscript{2}O framework.
However, our framework offers greater flexibility in designing robustness criteria while providing theoretical guarantees for general nonconvex objective functions.
For an in-depth discussion of various attack and defense strategies, we refer readers to the comprehensive surveys \cite{rodriguez2023survey, hallaji2024decentralized}.

\paragraph{Federated Learning and Swarm-based Optimization Methods.}
Inspired by swarm intelligence observed in nature, where agents collaborate to achieve a common goal, different swarm-based optimization methods have been proposed and studied in the literature, including particle swarm optimization (PSO) \cite{kennedy1995particle,kennedy1997particle}, ant colony optimization \cite{dorigo2005ant}, and consensus-based optimization (CBO) \cite{pinnau2017consensus,carrillo2018analytical,bailo2024cbx}, among others.
In recent years, mathematical foundations of interacting particle systems (IPS) have been solidified through the development of a rigorous analytical framework that leverages tools from stochastic and PDE analysis, see \cite{grassi2021mean,riedl2024perspective} and the references therein.

Federated learning (FL) \cite{mcmahan2017communication,konevcny2016federated,kairouz2021advances}, which naturally involves collaborative training in a distributed manner, shares a similar spirit with swarm-based optimization methods.
Several works have begun to explore FL problems from the perspective of swarm-based optimization.
For example, \cite{park2021fedpso} integrates PSO into the FL setting and proposes the FedPSO algorithm to reduce communication costs.
However, this work assumes homogeneous data and is limited to an attack-free scenario, which constrains its applicability to complex real-world FL applications.
To address the challenges of data heterogeneity and poisoning attacks,
\cite{fan2023cb} introduces a small shared global dataset among participants and develops a communication-efficient and Byzantine-robust distributed swarm learning (CB-DSL) framework by combining the principle of PSO with distributed gradient-based methods \cite{mcmahan2017communication}.
They also provide mathematical guarantees for CB-DSL to converge to stationary points in nonconvex settings, under assumptions about the relationship between local gradient updates and global exploration forces. 
For a comprehensive overview of distributed swarm learning, we refer readers to the surveys \cite{wang2024distributed, shammar2024swarm}.

Another notable line of research focuses on incorporating CBO into the FL setting.
The authors of \cite{carrillo2024fedcbo} propose a novel IPS called FedCBO to address decentralized clustered federated learning (DCFL) problems, accompanied by rigorous global convergence guarantees under mild assumptions on the objective functions.
However, FedCBO is designed for attack-free scenarios, leaving it vulnerable to adversarial attacks.
To remedy this, we incorporate in this paper a variant of CBO, called consensus-based bi-level optimization (CB\textsuperscript{2}O) \cite{trillos2024CB2O}, into the DCFL setting.
We propose a novel paradigm that enables benign agents to defend against malicious agents during training, mitigating the vulnerabilities of existing approaches, justified both theoretically and empirically.

The works mentioned above pave the way for connecting the fields of federated learning and swarm-based optimization, with benefits for both domains. On the one hand, well-developed mathematical tools from swarm-based optimization can help FL to establish solid theoretical foundations and inspire new algorithms.
Conversely, practical challenges arising in FL applications motivate and inspire the development of novel swarm-based optimization methods and pose new, mathematically intriguing questions.

\subsection{Organization}
In Section~\ref{sec:main}, we present the main theoretical contributions of this paper.
Specifically, we state in Theorem~\ref{thm:main} a global mean-field law convergence result for the CB\textsuperscript{2}O dynamics \eqref{eq:dyn_macro} and \eqref{eq:fokker_planck} in the presence of malicious agents who can perform a wide range of adversarial attacks. 
The accompanying proofs, presented in Sections~\ref{sec:rQQLP}\,--\,\ref{sec:proof_thm}, provide insights 
into how the choice of the hyperparameters $\beta$ and $\alpha$ makes CB\textsuperscript{2}O robust in adversarial settings.
Building upon this theoretical foundation, we turn towards a practical application of the CB\textsuperscript{2}O system in the setting of robust federated learning, where certain agents perform label-flipping (LF) attacks.
In Section~\ref{subsec:LF_attack} and \ref{subsec:FailFedCBO} we first revisit LF attacks in the context of decentralized clustered federated learning (DCFL) and illustrate the vulnerability of the FedCBO algorithm \cite{carrillo2024fedcbo} to such attacks.
Motivated by these observations, we introduce in Section~\ref{subsec:RobustFedCB2O}  FedCB\textsuperscript{2}O, a novel interacting particle system, which extends the CB\textsuperscript{2}O system to the DCFL setting.
We then adapt FedCB\textsuperscript{2}O in Section~\ref{subsec:FedCB2O_alg} to account for practical demands that arise in real-world FL applications and propose with Algorithm~\ref{alg:FedCB2O} a practical algorithm, whose efficiency is validated experimentally in a DCFL setting in the presence of malicious agents performing LF attacks.
Section \ref{sec:conclusion} concludes the paper.

For the sake of reproducible research we provide the code implementing the FedCB\textsuperscript{2}O algorithm proposed in this work and used to run the numerical experiments in Section   \ref{sec:experiments} in the GitHub repository \url{https://github.com/SixuLi/FedCB2O}.

\subsection{Notation}
Euclidean balls are denoted as \mbox{$B_{r}(\theta) := \{\tilde\theta \in \bbR^d\!:\! \Nnormal{\tilde\theta-\theta}_2 \leq r\}$}.
The distance between a point $\theta$ and a set $S\subset\R^d$ is $\dist(\theta, S) := \inf_{\tilde\theta \in S}\, \Nnormal{\tilde\theta-\theta}_2$, and the neighborhood of a set $S$ with radius $r$ is $\CN_r(S) := \{\tilde\theta \in \R^d : \dist(\tilde\theta, S) \leq r \}$.
We introduce $[N]:=\{1,2, \dots, N\}$, and $[N]\backslash j$ as the set $[N]$ excluding the element $j$.
The symbol $\propto$ is used to indicate equality up to a normalizing constant.
For the space of continuous functions~$f:X\rightarrow Y$ we write $\CC(X,Y)$, with $X\subset\bbR^n$ and a suitable topological space $Y$.
The operators $\nabla$ and $\Delta$ denote the gradient and Laplace operators of a function on~$\bbR^d$.
The main objects of study in this paper are laws of stochastic processes, $\rho\in\CC([0,T],\CP(\bbR^d))$, where the set $\CP(\bbR^d)$ contains all Borel probability measures over $\bbR^d$.
With $\rho_t\in\CP(\bbR^d)$ we refer to the snapshot of such law at time~$t$.
In case we refer to some fixed distribution, we write~$\indivmeasure$.
Measures~$\indivmeasure \in \CP(\bbR^d)$ with finite $p$-th moment $\int \Nnormal{\theta}_2^p\,d\indivmeasure(\theta)$ are collected in $\CP_p(\bbR^d)$.
For any $1\leq p<\infty$, $W_p$ denotes the \mbox{Wasserstein-$p$} distance between two Borel probability measures~$\indivmeasure_1,\indivmeasure_2\in\CP_p(\bbR^d)$ given by $W_p^p(\indivmeasure_1,\indivmeasure_2) = \inf_{\pi\in\Pi(\indivmeasure_1,\indivmeasure_2)}\int\Nnormal{\theta-\tilde\theta}_2^p\,d\pi(\theta,\tilde\theta)$,
where $\Pi(\indivmeasure_1,\indivmeasure_2)$ denotes the set of all couplings of $\indivmeasure_1$ and $\indivmeasure_2$.

\section{\texorpdfstring{Robustness of CB\textsuperscript{2}O Against Attacks}{Robustness of CB2O Against Attacks}}
\label{sec:main}

We now present and discuss the main theoretical result about the global convergence of the benign agent density~\eqref{eq:dyn_macro_b} and \eqref{eq:fokker_planck_b} of the mean-field CB\textsuperscript{2}O dynamics~\eqref{eq:dyn_macro} and \eqref{eq:fokker_planck} in mean-field law in the presence of a malicious agent density~\eqref{eq:dyn_macro_m} and \eqref{eq:fokker_planck_m} for a lower-level objective function~$L$ and an upper-level objective function~$G$ that satisfy the following assumptions.

\begin{assumption}\label{def:assumptions}
	Throughout we are interested in $L \in \CC(\bbR^d)$ and $G \in \CC(\bbR^d)$,
    for which
	\begin{enumerate}[label=A\arabic*,labelsep=10pt,leftmargin=35pt]
		\item\label{asm:thetaG} there exists a unique $\thetaG\in\Theta:=\argmin_{\theta\in\bbR^d} L(\theta)$ with $\underbar{L}:=L(\thetaG)=\inf_{\theta\in\bbR^d} L(\theta)$ such that $G(\thetaG) = \inf_{\theta^*\in\Theta} G(\theta^*)$,
        \item\label{asm:LipL} there exist $h_L,R^H_{L}>0$, and $H_L<\infty$ such that
        \begin{align}
            L(\theta)-\underbar{L}
            \leq H_L\N{\theta-\thetaG}_2^{h_L}
            \quad \text{ for all } \theta\in B_{R^H_{L}}(\thetaG),
        \end{align}
        \item\label{asm:icpL} there exist $L_\infty,R_L,\eta_L > 0$, and $\nu_L \in (0,\infty)$ such that
        \begin{subequations}
		\begin{align}
			\label{eq:icpL_1}
			\dist(\theta,\Theta)
            &\leq \frac{1}{\eta_L}\left(L(\theta)-\underbar{L}\right)^{\nu_L} \quad \text{ for all } \theta \in \CN_{R_L}(\Theta),\\
            \label{eq:icpL_2}
            L(\theta)-\underbar{L}
            &> L_\infty \quad \text{ for all } \theta \in \big(\CN_{R_L}(\Theta)\big)^c,
		\end{align}
        \end{subequations}
        \item\label{asm:LipG} there exist $h_G,R^H_{G} > 0$, and $H_G<\infty$ such that
        \begin{align}
            G(\theta)-G(\thetaG)
            \leq H_G\N{\theta-\thetaG}_2^{h_G}
            \quad \text{ for all } \theta\in B_{R^H_{G}}(\thetaG),
        \end{align}
        \item\label{asm:icpG} there exist $G_{\infty},R_G,\eta_G > 0$, and $\nu_G \in (0,\infty)$ such that for all $r_G\leq R_G$ there exists $\thetaGtilde\in B_{r_G}(\thetaG)$ such that
		\begin{subequations}
        \begin{align}
			\label{eq:icpG_1}
			\Nbig{\theta-\thetaGtilde}_2
            &\leq \frac{1}{\eta_G}\left(G(\theta)-G(\thetaGtilde)\right)^{\nu_G} \quad \text{ for all } \theta \in B_{r_G}(\thetaG),\\
			\label{eq:icpG_2}
			G(\theta)-G(\thetaGtilde)
            &> G_\infty \quad \text{ for all } \theta \in \CN_{r_G}(\Theta)\backslash B_{r_G}(\thetaG),
		\end{align}
        \end{subequations}
        \item\label{asm:growthG} there exist $k_G,R^K_G > 0$, and $K_G<\infty$ such that
        \begin{equation}\label{asm:quadratic_growth_G}
            G(\theta) - G(\thetaG) \geq K_G \N{\theta - \thetaG}_2^{k_G} \quad  \text{ for all } \theta \in \big(B_{R_G^K}(\thetaG) \big)^c \cap \CN_{R_G} (\Theta).
        \end{equation}
    \end{enumerate}
\end{assumption}
Assumptions~\ref{asm:thetaG}--\ref{asm:icpG} are identical to the ones in \cite[Assumption~2.6]{trillos2024CB2O} and we redirect readers to this paper for their discussion.
Assumption~\ref{asm:growthG} additionally imposes the growth of the upper-level objective function~$G$ in the farfield in a neighborhood of the set~$\Theta$ of global minimizers of the lower-level objective function~$L$.
This condition ensures that no regions in the farfield simultaneously exhibit small values for both $L$ and $G$.

We are now ready to state the main result about the global convergence of CB\textsuperscript{2}O in the presence of attacks as well as its robustness against those.
The proof is deferred to Section~\ref{sec:proof_thm} with auxiliary statements being presented in Sections~\ref{sec:rQQLP} and \ref{sec:attack_control}.
The proof framework follows the one of \cite{fornasier2021consensus,fornasier2021convergence,riedl2022leveraging,riedl2024perspective,trillos2024CB2O}.

\begin{theorem} [Convergence of the mean-field CB\textsuperscript{2}O dynamics~\protect{\eqref{eq:dyn_macro}} and~\protect{\eqref{eq:fokker_planck}} in the presence of attacks]
    \label{thm:main}
    Let $\varrho\in\CP(\bbR^d)$ be of the form $\varrho = w_b \varrho^b + w_m \varrho^m$ with $\varrho^b, \varrho^m\in\CP(\R^d)$ and $w_b, w_m \geq 0$ such that \mbox{$w_b + w_m = 1$}.
    Let $L \in \CC(\bbR^d)$ and $G \in \CC(\bbR^d)$ satisfy Assumptions~\ref{asm:thetaG}--\ref{asm:growthG}.
	Moreover, let $\rho^b_0 \in \CP_4(\bbR^d)$ be such that $\thetaG\in\supp{\rho^b_0}$.
    Fix any $\varepsilon \in (0,W_2^2(\rho^b_0,\delta_{\thetaG})/2)$ and $\vartheta \in (0,1)$,
    choose parameters $\lambda,\sigma > 0$ with $2\lambda > d\sigma^2 $, and
    define the time horizon
	\begin{equation}
    \label{eq:end_time_star_statement}
		T^* := \frac{1}{(1-\vartheta)\big(2\lambda-d\sigma^2\big)}\log\left({W_2^2(\rho^b_0,\delta_{\thetaG})/(2\varepsilon)}\right).
	\end{equation}
	Then,
	for $\delta_q>0$ in \eqref{eq:Q_beta} sufficiently small,
	there exist $\alpha_0 > 0$ and $\beta_0 > 0$, 
    depending (among problem-dependent quantities) on $R$, $\delta_q$, $\varepsilon$, $\vartheta$, and in particluar on $w_b$ and $w_m$,
    such that for all $\alpha > \alpha_0$ and $\beta < \beta_0$,
    if $\rho^m \in \CC([0,T^*], \CP_4(\bbR^d))$ and if $\rho^b \in \CC([0,T^*], \CP_4(\bbR^d))$ is a weak solution to the Fokker-Planck Equation~\eqref{eq:fokker_planck_b} on the time interval $[0,T^*]$ with initial condition $\rho^b_0$, and if the mapping $t \mapsto \mAlphaBeta{\rho_t}$ is continuous for $t\in[0,T^*]$,
    it holds 
    \begin{equation}
        \label{eq:end_time_statement}
        W_2^2(\rho^b_T,\delta_{\thetaG})/2 = \varepsilon
        \quad\text{ with }\quad
        T\in\left[\frac{1-\vartheta}{(1+\vartheta/2)}\;\!T^*,T^*\right].
    \end{equation}
	Furthermore, $W_2^2(\rho^b_t,\delta_{\thetaG})\leq W_2^2(\rho^b_0,\delta_{\thetaG}) \exp\big(\!-(1-\vartheta)\big(2\lambda- d\sigma^2\big)t\big)$ for all $t \in [0,T]$.
\end{theorem}

In the context of DFL,
this result can be understood as follows.
Despite the presence of malicious agents who attack and seek to interfere with the training dynamics,
the benign agents, provided that the number of agents is sufficiently large to make the conducted mean-field analysis descriptive,
still converge to the target model $\thetaG$ that minimizes the loss function $L$ while enjoying the robustness properties encoded through $G$. In order to achieve this, the CB\textsuperscript{2}O protocol merely needs to adjust the choice of the hyperparameters~$\beta$ and $\alpha$ to accommodate for the presence of malicious agents.
As described in detail in \eqref{eq:beta} and \eqref{eq:alpha}, the parameter $\beta$ needs to be readjusted to scale proportionally to the fraction of benign agents, i.e., $\beta\propto w_b$, and the parameter $\alpha$ needs to be increased by an amount of $\max\big\{0,\log\big(\frac{w_m}{w_b}R_G^K/\sqrt{\varepsilon}\big)\big\}$.
These adjustments of $\beta$ and $\alpha$ are reasonable, allowing in particular to recover the attack-free scenario as $w_m\rightarrow0$ and $w_b\rightarrow1$.

\revisedOne{For a comprehensive discussion of further technical aspects and facets of Theorem~\ref{thm:main}, we refer readers to the detailed comments after \cite[Theorem~2.7]{trillos2024CB2O}.}


Let us now present the proof for the mean-field global convergence result of CB\textsuperscript{2}O in the presence of attacks, Theorem~\ref{thm:main}, together with some auxiliary results.

\subsection{Robust Quantitative Quantiled Laplace Principle}
\label{sec:rQQLP}

We first provide an extension of the quantitative quantiled Laplace principle~\cite[Proposition~4.2]{trillos2024CB2O}, which takes into consideration the presence of attacks and carefully distills their influence.

\begin{proposition}[Robust quantitative quantiled Laplace principle]
	\label{prop:QQLP}
    Let $\varrho\in\CP(\bbR^d)$ be of the form $\varrho = w_b \varrho^b + w_m \varrho^m$ with $\varrho^b, \varrho^m\in\CP(\R^d)$ and $w_b, w_m \geq 0$ such that \mbox{$w_b + w_m = 1$}.
    Fix $\alpha>0$,
    let $r_G\in(0,\min\{R_G,R^H_{G},R_G^K,(\min\{G_{\infty},(\eta_GR_G^K)^{1/\nu_G}\}/(2H_G))^{1/h_G}\}]$ and $\delta_q\in(0,\min\{L_\infty,(\eta_Lr_G)^{1/\nu_L}\}/2]$.
    For any $r>0$ define $G_r:=\sup_{\theta\in B_{r}(\thetaG)}G(\theta)-G(\thetaG)$.
    Then, under
    the inverse continuity property~\ref{asm:icpL} on $L$,
    the H\"older continuity assumption~\ref{asm:LipL} on $L$,
    the inverse continuity property~\ref{asm:icpG} on $G$, the H\"older continuity assumption~\ref{asm:LipG} on $G$ and the growth condition~\ref{asm:growthG} on $G$,
    and provided that there exists
    \begin{equation}
        \label{eq:beta_QQLP}
    	\beta\in(0,1)
    	\quad\text{satisfying}\quad
    	q^L_{\beta}[\varrho]+\delta_q\leq\underbar{L}+\min\{L_\infty,(\eta_Lr_G)^{1/\nu_L}\},
    \end{equation}
    for any $r\in(0,\min\{r_R,r_G,R^H_{L},(\delta_q/H_L)^{1/h_L}\}]$ and for $u > 0$ such that 
    \[u + G_{r} + H_Gr_G^{h_G} \leq \min\{G_{\infty},(\eta_GR_G^K)^{1/\nu_G}\},\]
    we have  
    \begin{equation}
        \label{eq:lem:laplace_LG}
    \begin{split}
        &\N{\mAlphaBeta{\varrho} - \thetaG}_2
        \leq
        \frac{(u+G_r + H_Gr_G^{h_G})^{\nu_G}}{\eta_G} + \frac{\exp\left(-\alpha u\right)}{\varrho^b\big(B_r(\thetaG)\big)}\int_{\Qbeta{\varrho}}\!\Nbig{\theta\!-\!\thetaG}_2 d\varrho^b(\theta)\\
        &\;\;\;\;\,+\!\frac{w_m\exp\left(-\alpha u\right)}{w_b\varrho^b\big(B_r(\thetaG)\big)}\int_{\Qbeta{\varrho}\cap B_{\!R_G^K}(\thetaG)}\!\Nbig{\theta\!-\!\thetaG}_2 d\varrho^m(\theta)\\
        &\;\;\;\;\,+\!\frac{w_m\exp(\alpha G_r)}{w_b\varrho^b\big(B_r(\thetaG)\big)}\int_{\Qbeta{\varrho}\cap\big(B_{\!R_G^K}(\thetaG)\big)^{\!c}} \!\Nbig{\theta\!-\!\thetaG}_2 \exp\!\left(\!-\alpha K_G \N{\theta\!-\!\thetaG}_2^{k_G}\!\right) d\varrho^m(\theta).\!\!\!\!\!
    \end{split}
    \end{equation}
    We furthermore have $B_r(\thetaG)\subset\Qbeta{\varrho}\subset\CN_{r_G}(\Theta)$.
\end{proposition}

Let us first point out that the bound~\eqref{eq:lem:laplace_LG} in Proposition~\ref{prop:QQLP} reduces in the attack-free case, i.e., when $w_m=0$ (and thus $w_b=1$ as well as $\varrho=\varrho^b$), to \cite[Proposition~4.2]{trillos2024CB2O}.

\begin{remark}
    By distilling the influence of the malicious agents when proving the robust quantitative quantiled Laplace principle in Proposition~\ref{prop:QQLP},
    the result provides insights into how a reasonable attack to CB\textsuperscript{2}O needs to be designed to have any potential effect.
    The aim of an attack~$\varrho^m$ is to maximize the last two terms in \eqref{eq:lem:laplace_LG}.
    First and naturally, the bigger the portion~$w_m$ of the malicious agent density, the bigger is the potential of the attack.
    This is, however, provided that the attack is sophisticated, i.e., does not perform poorly w.r.t.\@ the lower-level objective function~$L$, given that the part of $\varrho^m$ that is supported outside of $\Qbeta{\varrho}$ has no influence on the magnitude of the aforementioned terms.
    Within the set $\Qbeta{\varrho}$ the adversaries' aim must be to suggest model parameters as distant as possible from $\thetaG$ but not farther than $R_G^K$ due to the growth of the upper-level objective function~$G$ outside of $B_{\!R_G^K}(\thetaG)$ that we imposed in Assumption~\ref{asm:growthG}.
    
    The label-flipping attack described in Section~\ref{subsec:LF_attack}, for instance, accomplishes precisely these goals.
    
    Any such attack, however, can only be of moderate impact, as we establish in Proposition~\ref{prop:attack} in Section~\ref{sec:attack_control} by providing upper bounds on the last two terms in \eqref{eq:lem:laplace_LG}, which eventually permits to counteract the attack with updated choices of the hyperparameters~$\beta$ and~$\alpha$.
\end{remark}

\begin{proof}[Proof of Proposition~\ref{prop:QQLP}]
    The proof follows the lines of \cite[Proposition~4.2]{trillos2024CB2O} while taking into account the presence of a malicious agents density~$\varrho^m$, distilling its influence, and exploiting the additional assumption~\ref{asm:growthG}.
    To keep the presentation below concise,
    we focus on those parts of the proof that differ from the one presented in \cite[Section~4.3]{trillos2024CB2O}. 
    We thus recommend the reader to follow in parallel the proof of \cite[Proposition~4.2]{trillos2024CB2O} to which we refer in several instances.

    \textbf{Preliminaries.}
    Following \cite{trillos2024CB2O}, we begin by establishing that $\qbetahalf{\varrho}\leq \frac{2}{\beta}\int_{\beta/2}^{\beta} \qa{\varrho}\,da\leq \qbeta{\varrho}$,
    as well as $B_r(\thetaG)\subset\Qbeta{\varrho}\subset\CN_{r_G}(\Theta)$ given that $L$ satisfies the H\"older continuity condition~\ref{asm:LipL} and the inverse continuity property~\ref{asm:icpL}, and in particular \eqref{eq:beta_QQLP}.
    With $r_G\leq R_G$, the inverse continuity property~\ref{asm:icpG} on $G$ holds in the set $\Qbeta{\varrho}$, in which $\mAlphaBetanoarg$ is computed.

    \textbf{Main proof.}
    In order to control the term $\Nbig{\mAlphaBeta{\varrho} - \thetaG}_2$,
    let us first recall that the measure $\varrho$ is of the form $\varrho = w_b \varrho^b + w_m \varrho^m$,
    which allows us to split the error using the definition of the consensus point $\mAlphaBeta{\varrho} = \int \theta \omegaa(\theta)/\Nnormal{\omegaa}_{L^1(\Ibeta{\varrho})}\,d\Ibeta{\varrho}(\theta)$ and Jensen's inequality:
    \begin{equation} \label{proof:Laplace_principle:eq:main_decomposition}
	\begin{split}
	    &\Nbig{\mAlphaBeta{\varrho} - \thetaG}_2
		\leq \int \Nbig{\theta-\thetaG}_2 \frac{\omegaa(\theta)}{\N{\omegaa}_{L^1(\Ibeta{\varrho})}} d\Ibeta{\varrho}(\theta)\\
        &\quad= w_b\int_{\Qbeta{\varrho}} \!\Nbig{\theta\!-\!\thetaG}_2 \frac{\omegaa(\theta)}{\N{\omegaa}_{L^1(\Ibeta{\varrho})}} d\varrho^b(\theta)\!+\!w_m\int_{\Qbeta{\varrho}} \!\Nbig{\theta\!-\!\thetaG}_2 \frac{\omegaa(\theta)}{\N{\omegaa}_{L^1(\Ibeta{\varrho})}} d\varrho^m(\theta).
	\end{split}
	\end{equation}
    \textbf{Contribution of benign agent density~$\varrho^b$ in \eqref{proof:Laplace_principle:eq:main_decomposition}.}
    Let us start with the first term in the last line of \eqref{proof:Laplace_principle:eq:main_decomposition}, for which we follow the steps taken in \cite{trillos2024CB2O}.
    Let $\widetilde{r}^b \geq r > 0$ and recall that $r\in(0,\min\{r_R,r_G,R^H_{L},(\delta_q/H_L)^{1/h_L}\}]$ by assumption.
	We can decompose
    \begin{equation}
        \label{proof:Laplace_principle:eq:decomposition_benign}
    \begin{split}
        &\int_{\Qbeta{\varrho}} \Nbig{\theta-\thetaG}_2 \frac{\omegaa(\theta)}{\N{\omegaa}_{L^1(\Ibeta{\varrho})}} d\varrho^b(\theta)\\
        &\qquad\qquad\qquad\quad\leq \int_{\Qbeta{\varrho}\cap B_{\widetilde{r}^b}(\thetaG)} \Nbig{\theta-\thetaG}_2 \frac{\omegaa(\theta)}{\N{\omegaa}_{L^1(\Ibeta{\varrho})}} d\varrho^b(\theta) \\
        &\qquad\qquad\qquad\quad\quad\,+ \int_{\Qbeta{\varrho}\cap \big(B_{\widetilde{r}^b}(\thetaG)\big)^{c}} \Nbig{\theta-\thetaG}_2 \frac{\omegaa(\theta)}{\N{\omegaa}_{L^1(\Ibeta{\varrho})}} d\varrho^b(\theta).
    \end{split}
    \end{equation}
    The first term in \eqref{proof:Laplace_principle:eq:decomposition_benign} is bounded by $\widetilde{r}^b$.
    Recalling the definition of $G_r$ and with the notation $\widetilde{G}_r:=\sup_{\theta\in B_{r}(\thetaG)}G(\theta)-G(\thetaGtilde)$, 
    we choose $\widetilde{r}^b = (u+\widetilde{G}_r)^{\nu_G}/{\eta_G}$ as in \cite{trillos2024CB2O}, which is a valid choice as it can be easily checked that $\widetilde{r}^b\geq r$; see \cite{trillos2024CB2O}.
    For the second term in \eqref{proof:Laplace_principle:eq:decomposition_benign}, we recall from \cite{trillos2024CB2O} that $\Nnormal{\omegaa}_{L^1(\Ibeta{\varrho})} \geq \exp\big(-\alpha (\widetilde{G}_r+G(\thetaGtilde))\big)\varrho\big(B_r(\thetaG)\big)$.
    With this we have
    \begin{equation*}
	\begin{split}
        &\int_{\Qbeta{\varrho}\cap \big(B_{\widetilde{r}^b}(\thetaG)\big)^{c}} \Nbig{\theta-\thetaG}_2 \frac{\omegaa(\theta)}{\N{\omegaa}_{L^1(\Ibeta{\varrho})}} d\varrho^b(\theta)\\
		&\quad\leq\int_{\Qbeta{\varrho}\cap \big(B_{\widetilde{r}^b}(\thetaG)\big)^{c}} \Nbig{\theta-\thetaG}_2 \frac{\exp\big(-\alpha (G(\theta)-(\widetilde{G}_r+G(\thetaGtilde)))\big)}{\varrho\big(B_r(\thetaG)\big)} d\varrho^b(\theta)\\
		&\quad\leq \frac{\exp\left(\!-\alpha \left(\inf_{\theta \in \Qbeta{\varrho}\cap \big(B_{\widetilde{r}^b}(\thetaG)\big)^{c}} \!G(\theta) \!-\! G(\thetaGtilde) \!-\! \widetilde{G}_r\right)\right)}{\varrho\big(B_r(\thetaG)\big)}\int_{\Qbeta{\varrho}}\!\Nbig{\theta\!-\!\thetaG}_2 d\varrho^b(\theta)\\
        &\quad\leq \frac{\exp\left(-\alpha u\right)}{w_b\varrho^b\big(B_r(\thetaG)\big)}\int_{\Qbeta{\varrho}}\Nbig{\theta-\thetaG}_2 d\varrho^b(\theta),
	\end{split}
	\end{equation*}
    where in the last step we first exploited that
    with $\widetilde{r}^b = (u+\widetilde{G}_r)^{\nu_G}/{\eta_G}$
    it holds, under the assumption $u + G_{r} + H_Gr_G^{h_G} \leq G_{\infty}$, that $\inf_{\theta \in (B_{\widetilde{r}^b}(\thetaG))^c\cap\Qbeta{\varrho}} G(\theta) - G(\thetaGtilde) - \widetilde{G}_r \geq u$ thanks to \ref{asm:icpG} as derived in \cite{trillos2024CB2O}.
    Secondly, we used that $\varrho(S)\geq w_b\varrho^b(S)$ for any set $S$.
    Since furthermore $\widetilde{r}^b\leq (u+G_r + H_Gr_G^{h_G})^{\nu_G}/\eta_G$ (see \cite{trillos2024CB2O}), we obtain for the first term in \eqref{proof:Laplace_principle:eq:main_decomposition} the upper bound    
    \begin{equation}
        \label{proof:Laplace_principle:eq:decomposition_benign_complete}
    \begin{split}
        & \int_{\Qbeta{\varrho}} \Nbig{\theta-\thetaG}_2 \frac{\omegaa(\theta)}{\N{\omegaa}_{L^1(\Ibeta{\varrho})}} d\varrho^b(\theta) \\
        &\qquad\qquad\leq \frac{(u+G_r + H_Gr_G^{h_G})^{\nu_G}}{\eta_G} + \frac{\exp\left(-\alpha u\right)}{w_b \varrho^b\big(B_r(\thetaG)\big)}\int_{\Qbeta{\varrho}}\Nbig{\theta-\thetaG}_2 d\varrho^b(\theta).
    \end{split}
    \end{equation}
    \textbf{Contribution of malicious agent density~$\varrho^m$ in \eqref{proof:Laplace_principle:eq:main_decomposition}.}
    Let us now tackle the second term in the last line of \eqref{proof:Laplace_principle:eq:main_decomposition}.
    Let $R_G^K \geq \widetilde{r}^m \geq r > 0$ and recall that $r\in(0,\min\{r_R,r_G,R^H_{L},(\delta_q/H_L)^{1/h_L}\}]$ by assumption.
	We can decompose
    \begin{equation}
        \label{proof:Laplace_principle:eq:decomposition_malicious}
    \begin{split}
        &\int_{\Qbeta{\varrho}} \Nbig{\theta-\thetaG}_2 \frac{\omegaa(\theta)}{\N{\omegaa}_{L^1(\Ibeta{\varrho})}} d\varrho^m(\theta)\\
        &\qquad\qquad\qquad\quad\leq \int_{\Qbeta{\varrho}\cap B_{\widetilde{r}^m}(\thetaG)} \Nbig{\theta-\thetaG}_2 \frac{\omegaa(\theta)}{\N{\omegaa}_{L^1(\Ibeta{\varrho})}} d\varrho^m(\theta) \\
        &\qquad\qquad\qquad\quad\quad\,+ \int_{\Qbeta{\varrho}\cap \big(B_{\widetilde{r}^m}(\thetaG)\big)^{c}} \Nbig{\theta-\thetaG}_2 \frac{\omegaa(\theta)}{\N{\omegaa}_{L^1(\Ibeta{\varrho})}} d\varrho^m(\theta).
    \end{split}
    \end{equation}
    Again, the first term in \eqref{proof:Laplace_principle:eq:decomposition_malicious} is bounded by $\widetilde{r}^m$.
    We can choose $\widetilde{r}^m = (u+\widetilde{G}_r)^{\nu_G}/{\eta_G}$ as before, which fulfills $\widetilde{r}^m\geq r$ as discussed already.
    Moreover, thanks to the assumption $u + G_{r} + H_Gr_G^{h_G} \leq (\eta_GR_G^K)^{1/\nu_G}$ it further holds
    \begin{equation*}
        \widetilde{r}^m
        = \frac{(u+\widetilde{G}_r)^{\nu_G}}{\eta_G}
        = \frac{(u+G_r + (G(\thetaGtilde)\!-\!G(\thetaG)))^{\nu_G}}{\eta_G}
        \leq \frac{(u+G_r + H_Gr_G^{h_G})^{\nu_G}}{\eta_G}
        \leq R_G^K.
    \end{equation*}
    In order to obtain the inequality in the next-to-last step be reminded that $\thetaGtilde\in B_{r_G}(\thetaG)$ and that $r_G\leq R^H_{G}$, allowing us to employ the H\"older continuity~\ref{asm:LipG} of $G$ to derive the upper bound $\absbig{G(\thetaGtilde)-G(\thetaG)}\leq H_G\Nbig{\thetaGtilde-\thetaG}_2^{h_G}\leq H_Gr_G^{h_G}$.
    For the second term in \eqref{proof:Laplace_principle:eq:decomposition_malicious} recall that $\Nnormal{\omegaa}_{L^1(\Ibeta{\varrho})} \geq \exp\big(-\alpha (\widetilde{G}_r+G(\thetaGtilde))\big)\varrho\big(B_r(\thetaG)\big)$ as before.
    With this we have
    \begin{equation}
        \label{proof:Laplace_principle:eq:decomposition_malicious_2}
	\begin{split}
        &\int_{\Qbeta{\varrho}\cap \big(B_{\widetilde{r}^m}(\thetaG)\big)^{c}} \Nbig{\theta-\thetaG}_2 \frac{\omegaa(\theta)}{\N{\omegaa}_{L^1(\Ibeta{\varrho})}} d\varrho^m(\theta)\\
		&\quad\leq\int_{\Qbeta{\varrho}\cap \big(B_{\widetilde{r}^m}(\thetaG)\big)^{c}} \Nbig{\theta-\thetaG}_2 \frac{\exp\big(-\alpha (G(\theta)-(\widetilde{G}_r+G(\thetaGtilde)))\big)}{\varrho\big(B_r(\thetaG)\big)} d\varrho^m(\theta)\\
		&\quad= \int_{\Qbeta{\varrho}\cap \big(B_{\widetilde{r}^m}(\thetaG)\big)^{c}\cap B_{R_G^K}(\thetaG)} \Nbig{\theta-\thetaG}_2 \frac{\exp\big(-\alpha (G(\theta)-(\widetilde{G}_r+G(\thetaGtilde)))\big)}{\varrho\big(B_r(\thetaG)\big)} d\varrho^m(\theta)\\
        &\quad\quad\,+\int_{\Qbeta{\varrho}\cap \big(B_{R_G^K}(\thetaG)\big)^{c}} \Nbig{\theta-\thetaG}_2 \frac{\exp\big(-\alpha (G(\theta)-(\widetilde{G}_r+G(\thetaGtilde)))\big)}{\varrho\big(B_r(\thetaG)\big)} d\varrho^m(\theta),
	\end{split}
	\end{equation}
    where in the last step, compared to the contribution from the benign agent density,
    we split the integral into two parts using the ball $B_{R_G^K}(\thetaG)$. Recall $\widetilde{r}^m\leq R_G^K$.
    For the first term in the last line of \eqref{proof:Laplace_principle:eq:decomposition_malicious_2}, we proceed analogously as before since the choice of $\widetilde{r}^m$ is identical.
    We thus obtain
    \begin{equation*}
	\begin{split}
        &\int_{\Qbeta{\varrho}\cap \big(B_{\widetilde{r}^m}(\thetaG)\big)^{c}\cap B_{R_G^K}(\thetaG)} \Nbig{\theta-\thetaG}_2 \frac{\exp\big(-\alpha (G(\theta)-(\widetilde{G}_r+G(\thetaGtilde)))\big)}{\varrho\big(B_r(\thetaG)\big)} d\varrho^m(\theta)\\
		&\quad\leq \frac{\exp\left(\!-\alpha \left(\inf_{\theta \in \Qbeta{\varrho}\cap \big(B_{\widetilde{r}^m}(\thetaG)\big)^{c}} \!G(\theta) \!-\! G(\thetaGtilde) \!-\! \widetilde{G}_r\right)\right)}{\varrho\big(B_r(\thetaG)\big)}\int_{\Qbeta{\varrho}\cap B_{R_G^K}(\thetaG)}\!\Nbig{\theta\!-\!\thetaG}_2 d\varrho^m(\theta)\\
        &\quad\leq \frac{\exp\left(-\alpha u\right)}{w_b\varrho^b\big(B_r(\thetaG)\big)}\int_{\Qbeta{\varrho}\cap B_{R_G^K}(\thetaG)}\Nbig{\theta-\thetaG}_2 d\varrho^m(\theta).
	\end{split}
	\end{equation*}
    For the second term in the last line of \eqref{proof:Laplace_principle:eq:decomposition_malicious_2}, on the other hand,
    we can compute
    \begin{equation*}
	\begin{split}
        &\int_{\Qbeta{\varrho}\cap \big(B_{R_G^K}(\thetaG)\big)^{c}} \Nbig{\theta-\thetaG}_2 \frac{\exp\big(-\alpha (G(\theta)-(\widetilde{G}_r+G(\thetaGtilde)))\big)}{\varrho\big(B_r(\thetaG)\big)} d\varrho^m(\theta)\\
		&\quad=\frac{\exp(\alpha G_r)}{\varrho\big(B_r(\thetaG)\big)}\int_{\Qbeta{\varrho}\cap\big(B_{R_G^K}(\thetaG) \big)^c} \!\Nbig{\theta\!-\!\thetaG}_2 \exp\big(\!-\!\alpha (G(\theta)\!-\!G(\thetaG))\big) \, d\varrho^m(\theta)\\
        &\quad\leq\frac{\exp(\alpha G_r)}{\varrho\big(B_r(\thetaG)\big)}\int_{\Qbeta{\varrho}\cap\big(B_{R_G^K}(\thetaG) \big)^c} \!\Nbig{\theta\!-\!\thetaG}_2 \exp\left(-\alpha K_G \N{\theta - \thetaG}_2^{k_G}\right) d\varrho^m(\theta)\\
        &\quad\leq\frac{\exp(\alpha G_r)}{w_b\varrho^b\big(B_r(\thetaG)\big)}\int_{\Qbeta{\varrho}\cap\big(B_{R_G^K}(\thetaG) \big)^c} \!\Nbig{\theta\!-\!\thetaG}_2 \exp\left(-\alpha K_G \N{\theta - \thetaG}_2^{k_G}\right) d\varrho^m(\theta),
	\end{split}
	\end{equation*}
    where we employed in the penultimate step that the growth condition~\ref{asm:growthG} on $G$ holds on the set $\big(B_{R_G^K}(\thetaG) \big)^c \cap \CN_{R_G} (\Theta)$ and $\Qbeta{\varrho} \subset \CN_{R_G} (\Theta)$.
    The last step simply uses again that $\varrho(S)\geq w_b\varrho^b(S)$ for any set $S$.
    Since, analogously to the above, we have $\widetilde{r}^m\leq (u+G_r + H_Gr_G^{h_G})^{\nu_G}/\eta_G$, we deduce for the second term in \eqref{proof:Laplace_principle:eq:main_decomposition} the upper bound
    \begin{equation}
        \label{proof:Laplace_principle:eq:decomposition_malicious_complete}
    \begin{split}
        & \int_{\Qbeta{\varrho}} \Nbig{\theta-\thetaG}_2 \frac{\omegaa(\theta)}{\N{\omegaa}_{L^1(\Ibeta{\varrho})}} d\varrho^m(\theta) \\
        &\quad\leq \frac{(u+G_r + H_Gr_G^{h_G})^{\nu_G}}{\eta_G} + \frac{\exp\left(-\alpha u\right)}{w_b\varrho^b\big(B_r(\thetaG)\big)}\int_{\Qbeta{\varrho}\cap B_{R_G^K}(\thetaG)}\Nbig{\theta-\thetaG}_2 d\varrho^m(\theta)\\
        &\quad\quad\,+\frac{\exp(\alpha G_r)}{w_b\varrho^b\big(B_r(\thetaG)\big)}\int_{\Qbeta{\varrho}\cap\big(B_{R_G^K}(\thetaG) \big)^c} \!\Nbig{\theta\!-\!\thetaGtilde}_2 \exp\left(-\alpha K_G \N{\theta - \thetaG}_2^{k_G}\right) d\varrho^m(\theta).
    \end{split}
    \end{equation}
    Collecting the estimates \eqref{proof:Laplace_principle:eq:decomposition_benign_complete} and \eqref{proof:Laplace_principle:eq:decomposition_malicious_complete}, multiplying them respectively by $w_b$ and $w_m$, which satisfy $w_b+w_m=1$, concludes the proof with \eqref{proof:Laplace_principle:eq:main_decomposition}.
\end{proof}

\subsection{Control of Attacks}
\label{sec:attack_control}

It remains to establish bounds on the last two terms appearing in \eqref{eq:lem:laplace_LG}, which are a result of the presence of the malicious agents density. 

\begin{proposition}[Control of attacks]
\label{prop:attack}
    Let $\varrho,\varrho^m\in\CP(\bbR^d)$. 
    Fix $\alpha \geq 1/\big(k_G K_G (R_G^K)^{k_G}\big)$.
    Moreover, let~
    \[r_G\in(0, \min\{R_G,R^H_{G},R_G^K,(\min\{G_{\infty},(\eta_GR_G^K)^{1/\nu_G},K_G (R_G^K)^{k_G}\}/(2H_G))^{1/h_G}\}]\] and 
    \[\delta_q\in(0,\min\{L_\infty,(\eta_Lr_G)^{1/\nu_L}\}/2] , .\]
    For any $r>0$ define $G_r:=\sup_{\theta\in B_{r}(\thetaG)}G(\theta)-G(\thetaG)$.
    Then, under
    the inverse continuity property~\ref{asm:icpL} on $L$,
    the H\"older continuity assumption~\ref{asm:LipL} on $L$
    and the growth condition~\ref{asm:growthG} on $G$,
    and provided that there exists
    \begin{equation}
        \label{eq:beta_QQLP2}
    	\beta\in(0,1)
    	\quad\text{satisfying}\quad
    	q^L_{\beta}[\varrho]+\delta_q\leq\underbar{L}+\min\{L_\infty,(\eta_Lr_G)^{1/\nu_L}\},
    \end{equation}
    for any $r\in(0,\min\{r_R,r_G,R^H_{L},(\delta_q/H_L)^{1/h_L}\}]$ and for $u > 0$ such that 
    \begin{equation*}
        u + G_{r} + H_Gr_G^{h_G} \leq \min\{G_{\infty},(\eta_GR_G^K)^{1/\nu_G},K_G (R_G^K)^{k_G}\},
    \end{equation*}
    we have
    \begin{equation}
    \label{eq:mali_bound_inside_ball}
    \begin{split}
        \sup_{\varrho^m\in\CP(\bbR^d)} \int_{\Qbeta{\varrho}\cap B_{R_G^K}(\thetaG)}\Nbig{\theta-\thetaG}_2 d\varrho^m(\theta)
        \leq R_G^K \sup_{\varrho^m\in\CP(\bbR^d)}  \varrho^m (\CN_{r_G}(\Theta))
        \leq R_G^K
    \end{split}
    \end{equation}
    as well as
    \begin{equation}
    \label{eq:mali_bound_outside_ball}
    \begin{split}
        &\!\!\exp(\alpha G_r)\!\!\sup_{\varrho^m\in\CP(\bbR^d)}\int_{\Qbeta{\varrho}\cap\big(B_{R_G^K}(\thetaG)\big)^{c}} \!\Nbig{\theta\!-\!\thetaG}_2 \exp\left(-\alpha K_G \N{\theta\!-\!\thetaG}_2^{k_G}\right) d\varrho^m(\theta)\!\!\\
        &\qquad\qquad\,\leq R_G^K \exp\left(-\alpha u\right) \sup_{\varrho^m\in\CP(\bbR^d)}\varrho^m (\CN_{r_G}(\Theta))
        \leq R_G^K \exp\left(-\alpha u\right).
    \end{split}
    \end{equation}
\end{proposition}

\begin{proof}
    \textbf{Preliminaries.}
    To begin with, we notice that the assumptions and in particular \eqref{eq:beta_QQLP2} allow us to show $\Qbeta{\varrho}\subset\CN_{r_G}(\Theta)$ as in the proof of Proposition~\ref{prop:QQLP}.

    \textbf{Term \eqref{eq:mali_bound_inside_ball}.}
    For any $\varrho^m\in\CP(\bbR^d)$ we have
    \begin{equation*}
    \begin{split}
        \int_{\Qbeta{\varrho}\cap B_{R_G^K}(\thetaG)}\Nbig{\theta-\thetaG}_2 d\varrho^m(\theta)
        &\leq \int_{\CN_{r_G} (\Theta)\cap B_{R_G^K}(\thetaG)}\Nbig{\theta-\thetaG}_2 d\varrho^m(\theta)\\
        &\leq \min \left\{R_G^K, \dist(\thetaG, \Theta) + r_G \right\} \varrho^m (\CN_{r_G}(\Theta)) \\
        &\leq R_G^K \varrho^m (\CN_{r_G}(\Theta)) \leq R_G^K.
    \end{split}
    \end{equation*}
    Taking the supremum over the measures $\varrho^m$ yields \eqref{eq:mali_bound_inside_ball}.

    \textbf{Term \eqref{eq:mali_bound_outside_ball}.}
    Observe that thanks to the choice $\alpha \geq 1/\big(k_G K_G (R_G^K)^{k_G}\big)$ the scalar function $f(x):=x\exp\!\big(\!-\!\alpha K_Gx^{k_G}\big)$ is non-increasing for $x\geq R_G^K$.
    In order to verify this,
    compute the derivative $f'(x)=(1-\alpha k_GK_G x^{k_G})\exp\!\big(\!-\!\alpha K_Gx^{k_G}\big)$ and notice that $f'(x)\leq0$ for $x\geq R_G^K$ with the choice of $\alpha$.
    Thus, for all $\theta \in \big(B_{R_G^K}(\thetaG)\big)^c$, it holds
    \begin{equation*}
        \Nbig{\theta-\thetaG}_2 \exp\left(-\alpha K_G \N{\theta-\thetaG}_2^{k_G}\right)
        \leq R_G^K \exp\left(-\alpha K_G (R_G^K)^{k_G}\right),
    \end{equation*}
    which allows us to derive for any $\varrho^m\in\CP(\bbR^d)$ that
    \begin{equation*}
    \begin{split}
        &\int_{\Qbeta{\varrho}\cap\big(B_{R_G^K}(\thetaG)\big)^{c}} \Nbig{\theta-\thetaGtilde}_2 \exp\left(-\alpha K_G \N{\theta-\thetaG}_2^{k_G}\right) d\varrho^m(\theta) \\
        &\qquad\qquad\,\leq R_G^K \exp\left(-\alpha K_G (R_G^K)^{k_G}\right) \varrho^m (\CN_{r_G}(\Theta))\\
        &\qquad\qquad\,\leq R_G^K \exp\left(-\alpha (u+G_r))\right) \varrho^m (\CN_{r_G}(\Theta))
        \leq R_G^K \exp\left(-\alpha (u+G_r))\right),
    \end{split}
    \end{equation*}
    where we used in the second step that by assumption $K_G (R_G^K)^{k_G}\geq u+G_r  + H_Gr_G^{h_G} \geq u+G_r$.
    Taking the supremum over the measures $\varrho^m$ yields \eqref{eq:mali_bound_outside_ball}.
\end{proof}

\subsection{Proof of Theorem~\ref{thm:main}}
\label{sec:proof_thm}

For the sake of convenience,
we introduce the notation 
\begin{equation}
	\label{eq:V}
    \CV(\rho^b_t)
    = \frac{1}{2}W_2^2\big(\rho^b_t,\delta_{\thetaG}\big)
    = \frac{1}{2}\int\N{\theta-\thetaG}_2^2d \rho^b_t(\theta),
\end{equation}
which is the quantity that we will analyze.

\begin{proof}[Proof of Theorem~\ref{thm:main}]
    The proof follows the same lines of \cite[Theorem~2.7]{trillos2024CB2O} for the measure~$\rho^b$ of benign agents while taking into account the presence of malicious agents~$\rho^m$.
    To keep its presentation below concise,
    we focus on those parts of the proof where it differs from the one presented in \cite[Section~4.5]{trillos2024CB2O} for the situation without malicious agents. 
    We thus recommend the reader to follow in parallel the proof of \cite[Theorem~2.7]{trillos2024CB2O} to which we refer in several instances.
    
    Let us start by recalling the definitions of $G_r$ and $c\left(\vartheta,\lambda,\sigma\right)$ from \cite{trillos2024CB2O}, and define, with the shorthand $\widetilde{G}_{\infty}:=\min\{G_{\infty},(\eta_GR_G^K)^{\frac{1}{\nu_G}},K_G (R_G^K)^{k_G}\}$,
    \begin{align*} 
        r_{G,\varepsilon}
        :=\min\left\{\left(\frac{1}{2H_G}\left(\eta_G\frac{c\left(\vartheta,\lambda,\sigma\right)\sqrt{\varepsilon}}{3}\right)^{\frac{1}{\nu_G}}\right)^{\frac{1}{h_G}},R_G,R^H_{G},R_G^K,\left(\frac{\widetilde{G}_{\infty}}{2H_G}\right)^{\frac{1}{h_G}}\right\}
    \end{align*}
    in line with the requirements of Propositions~\ref{prop:QQLP} and \ref{prop:attack}.
    We further emphasize that $\delta_q>0$ is sufficiently small in the sense that $\delta_{q}\leq \frac{1}{2} \min\{L_\infty,(\eta_Lr_{G,\varepsilon})^{1/\nu_L}\}$.

    \noindent
    \textbf{Choice of $\beta$.}
    With this choice, $\xi_{L,\varepsilon} := \min\{L_\infty,(\eta_Lr_{G,\varepsilon})^{1/\nu_L}\}-\delta_{q}$ fulfills $\xi_{L,\varepsilon}>0$.
    Define $r_{H,\varepsilon}:= \min\{R^H_{L}, ({\xi_{L,\varepsilon}}/{H_L})^{1/h_L}\}$, and choose $\beta\in(0,1)$ such that
    \begin{equation} \label{eq:beta}
        \beta <
        \beta_0
        := \frac{1}{2}w_b\rho^b_0\big(B_{r_{H,\varepsilon}/2}(\thetaG)\big) \exp(-p_{H,\varepsilon}T^*),
    \end{equation}
    where $p_{H,\varepsilon}$ is as defined in \cite[Proposition~4.4, Equation~(4.35)]{trillos2024CB2O} with $B=c\sqrt{\CV(\rho^b_0)}$ and with $r=r_{H,\varepsilon}$.
    Such choice of $\beta$ is possible since $\beta_0\in(0,1)$ as discussed in \cite{trillos2024CB2O} and since $w_b>0$.
    For such $\beta$ we yet again have $\qbeta{\rho_t} \leq \underbar{L}+\xi_{L,\varepsilon}$ for all $t\in[0,T^*]$ for the following reason:
    As in \cite{trillos2024CB2O}, $B_{r_{H,\varepsilon}}(\thetaG)\subset \{\theta:L(\theta)-\underbar{L}\leq\xi_{L,\varepsilon}\}$.
    Moreover, by \cite[Proposition~4.4]{trillos2024CB2O} with $r_{H,\varepsilon}$, $p_{H,\varepsilon}$ and $B$ (for $r$, $p$ and $B$) as defined before, it holds for all $t\in[0,T^*]$ that 
    \begin{align}
	\begin{aligned}
        w_b\rho^b_{t}\big(B_{r_{H,\varepsilon}}(\thetaG)\big)
		&\geq
        w_b\left(\int \phi_{r_{H,\varepsilon}}(\theta) \,d\rho^b_0(\theta)\right)\exp(-p_{H,\varepsilon}t) \\
		&\geq
        \frac{1}{2}w_b\,\rho^b_0\big(B_{r_{H,\varepsilon}/2}(\thetaG)\big) \exp(-p_{H,\varepsilon}T^*) 
		> \beta,
	\end{aligned}
	\end{align}
	where the last step is by choice of $\beta$.
    With this, the aforementioned set inclusion, and after recalling that thanks to $\rho_t=w_b\rho^b_t+w_m\rho^m_t$ it holds for any set~$B\subset\bbR^d$ that $w_b\rho^b_t(B)\leq\rho_t(B)$,
    we have for all $t\in[0,T^*]$ that
    \begin{equation}
        \beta
        <
        w_b\rho^b_{t}\big(B_{r_{H,\varepsilon}}(\thetaG)\big)
        \leq
        \rho_{t}\big(\{\theta:L(\theta)-\underbar{L}\leq\xi_{L,\varepsilon}\}\big)
    \end{equation}
    and thus, by definition of $\qbeta{\dummy}$ as the infimum, that $\qbeta{\rho_t}\leq\underbar{L}+\xi_{L,\varepsilon}$ for all $t\in[0,T^*]$.

    \noindent
    \textbf{Choice of $\alpha$.}
    Let us further define $u_\varepsilon := \frac{1}{4}\min\big\{\!\left(\eta_Gc(\vartheta,\lambda,\sigma)\sqrt{\varepsilon}/{3}\right)^{1/{\nu_G}}\!,{\widetilde{G}_{\infty}}\big\}>0$, and $\widetilde{r}_\varepsilon,r_\varepsilon>0$ as in \cite{trillos2024CB2O}, which satisfy with the identical argument $u_\varepsilon + G_{r_\varepsilon} + H_Gr_{G,\varepsilon}^{h_G}\leq 2u_\varepsilon + \widetilde{G}_{\infty}/2\leq \widetilde{G}_\infty$.
    
    With all parameters now in line with the requirements of Propositions~\ref{prop:QQLP} and \ref{prop:attack},
    it remains to choose $\alpha$
    such that $\alpha > 
		\alpha_0$ with
	\begin{equation}
	   \label{eq:alpha}
	\begin{split}
		\alpha_0
		&:= \max\Bigg\{\frac{1}{k_G K_G (R_G^K)^{k_G}},\frac{1}{u_\varepsilon}\Bigg(\log\left(\frac{12}{c(\vartheta,\lambda,\sigma)}\right)\!-\!\log\left(\rho^b_{0}\big(B_{r_\varepsilon/2}(\thetaG)\big)\right)\\
        &\qquad\quad+\!\max\left\{\frac{1}{2},\frac{p_\varepsilon}{(1-\vartheta)\big(2\lambda-d\sigma^2\big)}\right\}\log\left(\frac{\CV(\rho^b_0)}{\varepsilon}\right)\!+\!\max\left\{0,\log\left(\frac{w_m}{w_b}\frac{R_G^K}{\sqrt{\varepsilon}}\right)\right\}\!\Bigg)\Bigg\},
	\end{split}
	\end{equation}
    where $p_{\varepsilon}$ is as $p$ defined in \cite[Proposition~4.4, Equation~(4.35)]{trillos2024CB2O} with  $B=c(\vartheta,\lambda,\sigma)\sqrt{\CV(\rho^b_0)}$ and with $r=r_\varepsilon$.

    \noindent
    \textbf{Main proof.}
    Let us now define the time horizon $T_{\alpha,\beta} \geq 0$, which may depend on $\alpha$ and $\beta$,
    by
	\begin{align} \label{eq:endtime_T}
		T_{\alpha,\beta}
        := \sup\big\{t\geq0 : \CV(\rho^b_{t'}) > \varepsilon \text{ and } \N{\mAlphaBeta{\rho_t}-\thetaG}_2 < C(t') \text{ for all } t' \in [0,t]\big\}
	\end{align}
	with $C(t):=c(\vartheta,\lambda,\sigma)\sqrt{\CV(\rho^b_t)}$.
	Notice for later use that $C(0)=B$.

    Our aim is to show $\CV(\rho^b_{T_{\alpha,\beta}}) = \varepsilon$ with $T_{\alpha,\beta}\in\big[\frac{1-\vartheta}{(1+\vartheta/2)}\;\!T^*,T^*\big]$ and that we have at least exponential decay of $\CV(\rho^b_t)$ until time $T_{\alpha,\beta}$, i.e., until accuracy~$\varepsilon$ is reached.
	
By the continuity of the mappings~$t\mapsto\CV(\rho^b_{t})$ and~$t\mapsto\mAlphaBeta{\rho_{t}}$, it follows that ${T_{\alpha,\beta}>0}$, since $\CV(\rho^b_{0}) > \varepsilon$ and $\Nbig{\mAlphaBeta{\rho_{0}} - \thetaG}_2 < C(0)$.
	While the former is immediate by assumption, for the latter, an application of Propositions~\ref{prop:QQLP} and~\ref{prop:attack} with $r_{G,\varepsilon}$, $r_\varepsilon$, $u_\varepsilon$ and $\rho_0$ yields
    \begin{align} \label{eq:proof:lapl1}
	\begin{split}
        &\N{\mAlphaBeta{\rho_{0}} - \thetaG}_2
        \leq
        \frac{(u_\varepsilon+G_{r_\varepsilon} + H_Gr_{G,\varepsilon}^{h_G})^{\nu_G}}{\eta_G}+\\
        &\quad\,\quad\, + \frac{\exp\left(-\alpha u_\varepsilon\right)}{\rho_0^b\big(B_{r_\varepsilon}(\thetaG)\big)}\int_{\Qbeta{\rho_0}}\N{\theta-\thetaG}_2d\rho^b_{0}(\theta)+\!\frac{2w_m\exp\left(-\alpha u_\varepsilon\right)}{w_b\rho_0^b\big(B_{r_\varepsilon}(\thetaG)\big)}R_G^K\\
        &\quad\,\leq \frac{c\left(\vartheta,\lambda,\sigma\right)\sqrt{\varepsilon}}{3} + \frac{\exp\left(-\alpha u_\varepsilon\right)}{\rho_0^b\big(B_{r_\varepsilon}(\thetaG)\big)}\sqrt{2\CV(\rho^b_0)} +\frac{2w_m\exp\left(-\alpha u_\varepsilon\right)}{w_b\rho_0^b\big(B_{r_\varepsilon}(\thetaG)\big)}R_G^K\\
        &\quad\,\leq c\left(\vartheta,\lambda,\sigma\right)\sqrt{\varepsilon}
        < c\left(\vartheta,\lambda,\sigma\right)\sqrt{\CV(\rho^b_0)}
        = C(0),
	\end{split}
	\end{align}
    where the first step in the last line holds by choice of $\alpha$ in \eqref{eq:alpha}.

     Next, we show that the functional $\CV(\rho^b_t)$ decays up to time $T_{\alpha,\beta}$
    \begin{enumerate}[label=(\roman*),labelsep=10pt,leftmargin=35pt,topsep=2pt]
        \item at least exponentially fast (with rate $(1-\vartheta)(2\lambda-d\sigma^2)$), and \label{enumerate:proof:atleastexpdecay}
        \item at most exponentially fast (with rate $(1+\vartheta/2)(2\lambda-d\sigma^2)$).\label{enumerate:proof:atmostexpdecay}
    \end{enumerate}
	To obtain \ref{enumerate:proof:atleastexpdecay}, recall that \cite[Lemma~4.1]{trillos2024CB2O} provides an upper bound on $\frac{d}{dt}\CV(\rho^b_t)$ given by
    \begin{equation}
	\begin{split}
	    \frac{d}{dt}\CV(\rho^b_t)
		&\leq
		-\left(2\lambda - d\sigma^2\right) \CV(\rho^b_t)
	    + \sqrt{2}\left(\lambda + d\sigma^2\right) \sqrt{\CV(\rho^b_t)} \N{\mAlphaBeta{\rho_t}-\thetaG}_2\phantom{.} \\
	    &\quad\,+ \frac{d\sigma^2}{2} \N{\mAlphaBeta{\rho_t}-\thetaG}_2^2\\
        &\leq -(1-\vartheta)\left(2\lambda-d\sigma^2\right)\CV(\rho^b_t)
		\quad \text{ for all } t \in (0,T_{\alpha,\beta}),
	\end{split}
	\end{equation}
	where the last step follows from the definition of $T_{\alpha,\beta}$ in \eqref{eq:endtime_T} by construction.
    Analogously, for \ref{enumerate:proof:atmostexpdecay}, by the second part of \cite[Lemma~4.1]{trillos2024CB2O},
    we obtain a lower bound on $\frac{d}{dt}\CV(\rho^b_t)$ of the form
    \begin{equation}
    \begin{split}
        \frac{d}{dt}\CV(\rho^b_t)
        &\geq
        -\left(2\lambda - d\sigma^2\right) \CV(\rho^b_t)
        - \sqrt{2}\left(\lambda + d\sigma^2\right) \sqrt{\CV(\rho^b_t)} \N{\mAlphaBeta{\rho_t}-\thetaG}_2 \\
        &\geq -(1+\vartheta/2)\left(2\lambda - d\sigma^2\right) \CV(\rho^b_t)
        \quad \text{ for all } t \in (0,T_{\alpha,\beta}),
    \end{split}
    \end{equation}
    where the second inequality again exploits the definition of $T_{\alpha,\beta}$.
	Gr\"onwall's inequality now implies for all $t \in [0,T_{\alpha,\beta}]$ the upper and lower bounds
	\begin{align}
		\CV(\rho^b_0) \exp\left(- (1+\vartheta/2)\left(2\lambda-d\sigma^2\right) t\right)
        \leq
        \CV(\rho^b_t)
		\leq
        \CV(\rho^b_0) \exp\left(- (1-\vartheta)\left(2\lambda-d\sigma^2\right) t\right), \label{eq:evolution_J}
	\end{align}
    i.e., \ref{enumerate:proof:atleastexpdecay} and \ref{enumerate:proof:atmostexpdecay}.
	As in \cite{trillos2024CB2O}, $\max_{t \in [0,T_{\alpha,\beta}]} \Nbig{\mAlphaBeta{\rho_{t}} - \thetaG}_2 \leq \max_{t \in [0,T_{\alpha,\beta}]} C(t)\leq C(0)$.
    
    To conclude, it remains to prove that $\CV(\rho^b_{T_{\alpha,\beta}}) = \varepsilon$ with $T_{\alpha,\beta}\in\big[\frac{1-\vartheta}{(1+\vartheta/2)}\;\!T^*,T^*\big]$.
    For this we distinguish the following three cases.

    \noindent
	\textbf{Case $T_{\alpha,\beta} \geq T^*$:}
	We can use the definition of $T^*$ in \eqref{eq:end_time_star_statement} and the time-evolution bound of $\CV(\rho^b_t)$ in \eqref{eq:evolution_J} to conclude that $\CV(\rho^b_{T^*}) \leq \varepsilon$.
	Hence, by definition of $T_{\alpha,\beta}$ in \eqref{eq:endtime_T} together with the continuity of $t\mapsto\CV(\rho^b_t)$, we find $\CV(\rho^b_{T_{\alpha,\beta}}) =\varepsilon$ with $T_{\alpha,\beta} = T^*$.
	
	\noindent
	\textbf{Case $T_{\alpha,\beta} < T^*$ and $\CV(\rho^b_{T_{\alpha,\beta}}) \leq \varepsilon$:}
	By continuity of $t\mapsto\CV(\rho^b_t)$, it holds for $T_{\alpha,\beta}$, $\CV(\rho^b_{T_{\alpha,\beta}}) = \varepsilon$.
    Thus, $\varepsilon
        = \CV(\rho^b_{T_{\alpha,\beta}})
        \geq \CV(\rho^b_0) \exp\!\big(\!- (1+\vartheta/2)\big(2\lambda-d\sigma^2\big) T_{\alpha,\beta}\big)$ by \eqref{eq:evolution_J}, or reordered
    \begin{align}
        \frac{1-\vartheta}{(1+\vartheta/2)} \, T^*
        =\frac{1}{(1+\vartheta/2)\left(2\lambda-d\sigma^2\right)}\log\left(\frac{\CV(\rho^b_0)}{\varepsilon}\right)
        \leq T_{\alpha,\beta}
        < T^*.
    \end{align}
    \noindent
	\textbf{Case $T_{\alpha,\beta} < T^*$ and $\CV(\rho^b_{T_{\alpha,\beta}}) > \varepsilon$:}
	We shall show that this case can never occur by verifying that $\Nbig{\mAlphaBeta{\rho_{T_{\alpha,\beta}}} - \thetaG}_2 < C(T_{\alpha,\beta})$ due to the choices of $\alpha$ in~\eqref{eq:alpha} and $\beta$ in \eqref{eq:beta}.
	In fact, fulfilling simultaneously both $\CV(\rho^b_{T_{\alpha,\beta}})>\varepsilon$ and $\Nbig{\mAlphaBeta{\rho_{T_{\alpha,\beta}}} - \thetaG}_2 < C(T_{\alpha,\beta})$ would contradict the definition of $T_{\alpha,\beta}$ in \eqref{eq:endtime_T} itself.
	To this end, we apply again Propositions~\ref{prop:QQLP} and~\ref{prop:attack} with $r_{G,\varepsilon}$, $r_\varepsilon$, $u_\varepsilon$ and obtain
    \begin{align} \label{eq:proof:lapl2}
	\begin{split}
        &\N{\mAlphaBeta{\rho_{T_{\alpha,\beta}}} - \thetaG}_2
        \leq
        \frac{(u_\varepsilon+G_{r_\varepsilon} + H_Gr_{G,\varepsilon}^{h_G})^{\nu_G}}{\eta_G}+\\
        &\quad\,\quad\, +\frac{\exp\left(-\alpha u_\varepsilon\right)}{\rho_{T_{\alpha,\beta}}^b\big(B_{r_\varepsilon}(\thetaG)\big)}\int_{\Qbeta{\rho_{T_{\alpha,\beta}}}}\N{\theta\!-\!\thetaG}_2d\rho^b_{T_{\alpha,\beta}}(\theta)\!+\!\frac{2w_m\exp\left(-\alpha u_\varepsilon\right)}{w_b\rho_{T_{\alpha,\beta}}^b\big(B_{r_\varepsilon}(\thetaG)\big)}R_G^K\\
        &\quad\,\leq \frac{c\left(\vartheta,\lambda,\sigma\right)\sqrt{\varepsilon}}{3}\!+\!\frac{\exp\left(-\alpha u_\varepsilon\right)}{\rho_{T_{\alpha,\beta}}^b\big(B_{r_\varepsilon}(\thetaG)\big)}\left(\int_{\Qbeta{\rho_{T_{\alpha,\beta}}}}\N{\theta\!-\!\thetaG}_2d\rho^b_{T_{\alpha,\beta}}(\theta) \!+\! \frac{2w_m}{w_b}R_G^K\right)\\
        &\quad\,< \frac{c\left(\vartheta,\lambda,\sigma\right)\sqrt{\CV(\rho^b_{T_{\alpha,\beta}})}}{3}+\frac{\exp\left(-\alpha u_\varepsilon\right)}{\rho_{T_{\alpha,\beta}}^b\big(B_{r_\varepsilon}(\thetaG)\big)}\left(\sqrt{\CV(\rho^b_{T_{\alpha,\beta}})} + \frac{2w_m}{w_b}R_G^K\right),
	\end{split}
	\end{align}
    where for the last step we recall that in this case we assumed $\varepsilon<\CV(\rho^b_{T_{\alpha,\beta}})$.
    Since it holds for $B=C(0)$, $\max_{t \in [0,T_{\alpha,\beta}]}\Nbig{\mAlphaBeta{\rho_{t}} - \thetaG}_2 \leq B$, \cite[Proposition~4.4]{trillos2024CB2O} guarantees that there exists a $p_\varepsilon>0$ not depending on $\alpha$ (but depending on $B$ and $r_\varepsilon$) with
	\begin{align}
	\begin{aligned}
	    \rho^b_{T_{\alpha,\beta}}(B_{r_\varepsilon}(\thetaG))
		\geq \left(\int \phi_{r_\varepsilon} \,d\rho^b_0\right)\exp(-p_\varepsilon T_{\alpha,\beta}) 
		\geq \frac{1}{2}\rho^b_0\big(B_{r_\varepsilon/2}(\thetaG)\big) \exp(-p_\varepsilon T^*) 
		> 0,
	\end{aligned}
	\end{align}
	where we used $\thetaG\in\supp{\rho^b_0}$ for bounding the initial mass $\rho_0$ together with $T_{\alpha,\beta}\leq T^*$.
	With this we can continue the chain of inequalities in~\eqref{eq:proof:lapl2} to obtain
	\begin{align} \label{eq:proof:lapl22}
	\begin{split}
		&\N{\mAlphaBeta{\rho_{T_{\alpha,\beta}}} - \thetaG}_2
		< \frac{c\left(\vartheta,\lambda,\sigma\right)\sqrt{\CV(\rho^b_{T_{\alpha,\beta}})}}{3}\\
        &\quad\,\quad\,+\frac{2\exp\left(-\alpha u_\varepsilon\right)}{\rho_{0}^b\big(B_{r_\varepsilon/2}(\thetaG)\big)\exp(-p_\varepsilon T^*) }\left(\sqrt{\CV(\rho^b_{T_{\alpha,\beta}})} + \frac{2w_m}{w_b}R_G^K\right)\\
		&\quad\,\leq \frac{2c\left(\vartheta,\lambda,\sigma\right)\sqrt{\CV(\rho^b_{T_{\alpha,\beta}})}}{3} + \frac{c\left(\vartheta,\lambda,\sigma\right)\sqrt{\varepsilon}}{3}
        < c\left(\vartheta,\lambda,\sigma\right)\sqrt{\CV(\rho^b_{T_{\alpha,\beta}})}
		= C(T_{\alpha,\beta}),
	\end{split}
	\end{align}
    where the first inequality in the last line holds by choice of $\alpha$ in \eqref{eq:alpha} and the second since in this case $\varepsilon<\CV(\rho^b_{T_{\alpha,\beta}})$.
	This establishes again a contradiction.
\end{proof}

\section{\texorpdfstring{Robustness of FedCB\textsuperscript{2}O Against Label-Flipping Attacks in Decentralized Clustered Federated Learning}{Robustness of FedCB2O Against Label-Flipping Attacks in Decentralized Clustered Federated Learning}}\label{sec:algorithm}

In Section~\ref{subsec:LF_attack}, we describe the decentralized clustered federated learning (DCFL) setting considered in \cite{onoszko2021decentralized,beltran2023decentralized,carrillo2024fedcbo} and review label-flipping (LF) attacks~\cite{fung2020limitations, tolpegin2020data, jebreel2023fl, JEBREEL2024111} within this context.
We then revisit in Section~\ref{subsec:FailFedCBO} the FedCBO algorithm \cite{carrillo2024fedcbo}
and explain its vulnerability to LF attacks.
This motivates and leads, as we describe in Section~\ref{subsec:RobustFedCB2O}, to the development of the FedCB\textsuperscript{2}O system, an adaptation of CB\textsuperscript{2}O~\cite{trillos2024CB2O} to the DCFL problem.
In Section~\ref{subsec:FedCB2O_alg}, we then present the FedCB\textsuperscript{2}O algorithm (Algorithm~\ref{alg:FedCB2O}),
an implementation of the core principles of the FedCB\textsuperscript{2}O system that addresses the practical challenges encountered in real-world FL applications.
Finally, Section~\ref{sec:experiments} showcases the effectiveness of the FedCB\textsuperscript{2}O algorithm in practical scenarios by providing an extensive empirical study where we compare our algorithm's performance with those of baseline methodologies for the DCFL setting in the presence of malicious agents performing label-flipping attacks.

\subsection{Label-Flipping Attacks in Decentralized Clustered Federated Learning}\label{subsec:LF_attack}

In decentralized clustered federated learning problems \cite{onoszko2021decentralized,carrillo2024fedcbo},
each agent is assumed to belong to one of $K$ non-overlapping groups denoted by $S_1, \dots, S_K$.
An agent from group $S_k$ possesses data points generated from a distribution $\mathcal{D}_k$, which can be used to train the agent's own local model.
Denoting by $\ell \left(\theta; z \right): \Theta \rightarrow \mathbb{R}$ the loss function associated with a data point $z$, where $\Theta \subset \mathbb{R}^d$ is the parameter space of the learning models,
our goal is to minimize the population loss
\begin{equation}
    \label{eq:population_loss}
    L_k(\theta) := \mathbb{E}_{z \sim \mathcal{D}_k} \left[\ell\left(\theta; z \right) \right]
\end{equation}
simultaneously for all $k \in [K]$ under the data privacy constraints of FL. 
In other words, we wish to find for all loss functions~$L_k$ minimizers
\begin{equation}\label{eq:CFL_problem}
    \theta^{*, k} \in  \argmin_{\theta \in \Theta} L_k(\theta)
\end{equation}
without breaching the privacy protocol.
As pointed out in Remark~\ref{rem:DCFL}, the losses $L_k$ are the lower-level objective functions of the individual agents which depend on their group affiliation.

We consider in what follows $C$-class classification problems. In particular, each data point $z$ is of the form $z = (x,y) \in \CX \times \CY$,
where $\CX \subset \R^n$ is the data feature space and $\CY := \{1, \dots, C\}$ the label set.

The decentralized nature of the training process in DCFL increases the vulnerability of models to attacks from malicious agents.
One easy-to-implement but efficient and stealthy attack is the label-flipping (LF) attack, which was first introduced in the setting of centralized machine learning problems \cite{Biggio2012PoisoningAA, steinhardt2017certified}, and later studied in the context of distributed learning \cite{fung2020limitations,tolpegin2020data, jebreel2023fl,JEBREEL2024111}.
The goal of malicious agents is to poison the system such that trained models of benign agents incorrectly predict for samples from a source class with label $c_{S} \in \CY$ the target label $c_{T} \in \CY$.
To achieve this, attackers select in their own local datasets those samples with label $c_S$ and then flip their labels to the label $c_{T}$ while leaving the data features unchanged. They then train their models on the poisoned local datasets to obtain poisoned models, which they share with other participants in the DCFL system.
We illustrate the working principle of an LF attack in the DCFL setting in Figure~\ref{fig:LF_attack}.
\begin{figure}[htb]
\centering
\includegraphics[trim=12 102 30 146,clip,width=0.95\textwidth]{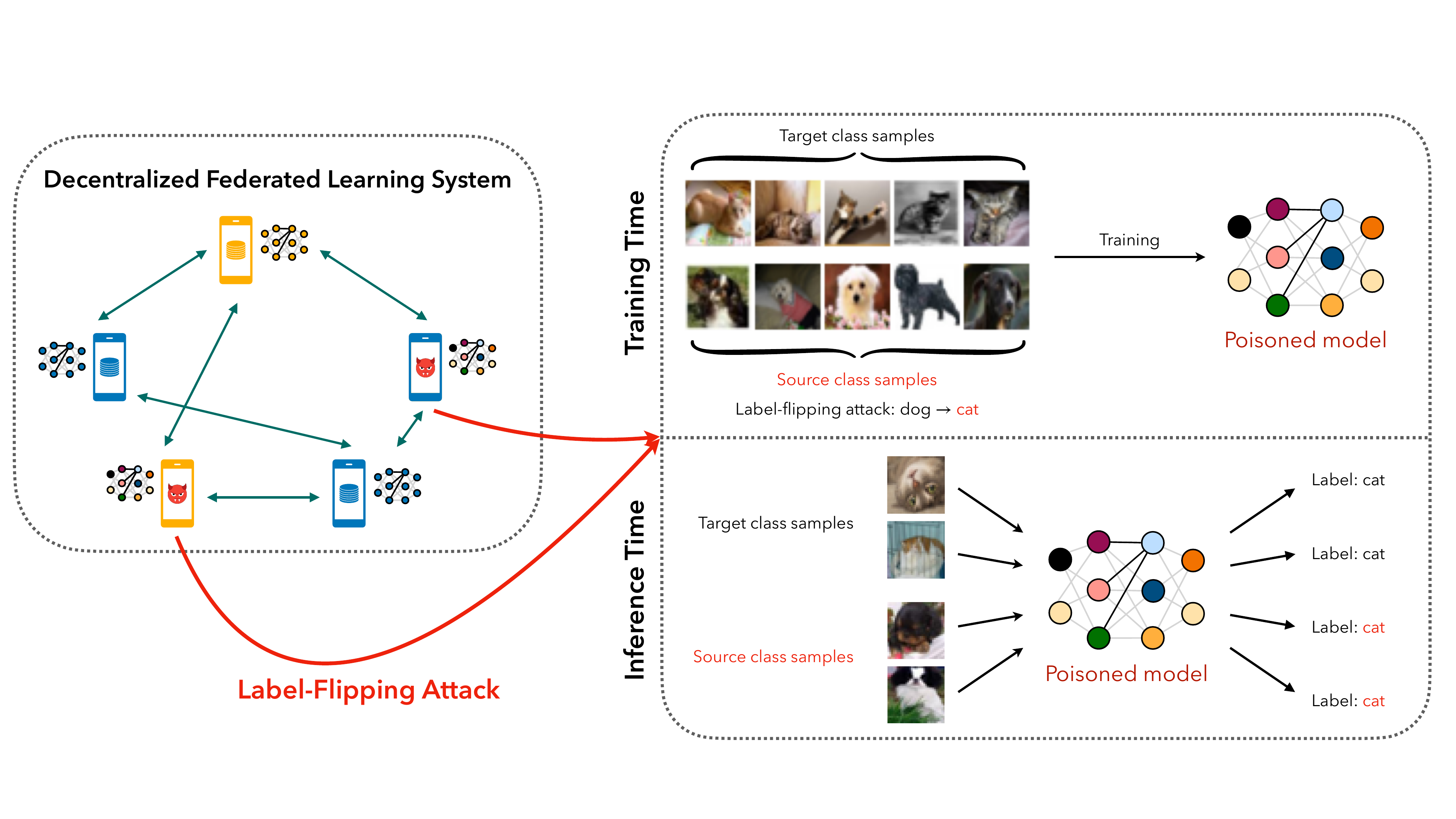}
\caption{An illustration of malicious agents performing a label-flipping attack in a decentralized clustered federated learning system.
The malicious agents flip in their own local dataset the labels of source class samples ($c_S =$ ``dogs'') to a target label ($c_T =$``cat'') while keeping the images themselves unchanged, before they train their models on the poisoned local dataset.
During the communication step, they share the poisoned models with the entire system.
} \label{fig:LF_attack}
\end{figure}

Mathematically, we can formulate the goals of the benign and malicious agents as two different optimization problems.
Benign agents in group $S_k$, $k \in [K]$, aim to minimize the population loss $L_k$ as defined in \eqref{eq:population_loss}.
For the sake of comparison, let us rewrite $L_k$ as a sum of class-wise losses $L_{k,c}$, i.e.,
\begin{equation}\label{eq:benignAgentObj}
    L_k(\theta) = \sum_{c=1}^{C}  w_{k,c} L_{k,c}(\theta), 
\end{equation}
where
\begin{equation}
    w_{k,c} := \mathbb{P}_{\CD_k} (y = c)\quad\text{and}\quad  L_{k,c}(\theta) := \E_{\{(x,y) \sim \mathcal{D}_k | y = c  \}} \left[\ell(\theta;(x,c)) \right] \quad \text{for } c \in [C].
\end{equation}
In contrast, to perform a LF attack, malicious agents intend to solve the poisoned problem
\begin{equation}\label{eq:MaliAgentObj}
    \argmin_{\theta \in \R^d}\; L^m_k(\theta) := \sum_{c\neq c_S} w_{k,c} L_{k,c}(\theta) + w_{k,c_{S}} \E_{\{(x,y) \sim \CD_k | y = {\blue {c_S}}\}} \left[ \ell \left( \theta; (x, {\red {c_T}}) \right) \right]. 
\end{equation}
The first term in \eqref{eq:MaliAgentObj} means that malicious agents do not alter samples other than the ones from the source class,
whereas the second term indicates that malicious agents flip all labels of samples from source class $c_S$ to target label $c_T$ while keeping the data features unchanged.

The LF attack significantly degrades the performance of trained models on the source class~$c_S$, while not affecting their performance on other classes.
This characteristic makes LF attacks generally challenging to detect, in particular when $C$ is large and only a few classes are attacked. We can observe this phenomenon in the experimental Section~\ref{sec:experiments}, see Table \ref{tab:baselines}.

\subsection{Vulnerability of FedCBO to Label-Flipping Attacks}
\label{subsec:FailFedCBO}

In this section, we revisit the FedCBO system proposed in \cite{carrillo2024fedcbo}, which was designed for the DCFL paradigm in the idealized setting where no malicious agents are present, and discuss the reasons for its vulnerability to LF attacks.

Let us therefore consider, without loss of generality, the DCFL setting from Section \ref{subsec:LF_attack} with $K=2$ clusters. 
We assume that all $N_1$ agents in cluster $1$ share the same loss function $L_1$, while all $N_2$ agents in cluster $2$ have another loss function $L_2$.
The associated positions of the particles, which, in the context of DCFL, correspond to the model parameters of the agents, 
are denoted by $\{\theta_t^{1, i_1}\}_{i_1=1}^{N_1} \subset \R^d$ for the agents from cluster $1$ and by $\{\theta_t^{2,i_2}\}_{i_2=1}^{N_2} \subset \R^d$ for the agents from cluster $2$.
To collaboratively minimize the objective functions $L_1$ and $L_2$ simultaneously while being oblivious to the cluster identities of the other agents,
the authors of \cite{carrillo2024fedcbo} propose to employ the FedCBO system,
which describes the dynamics of the collection of the $N = N_1 + N_2$ interacting particles
in terms of the SDE system
\begin{subequations}\label{eq:fedcbo_dynamic}
\begin{equation}\label{eq:fedcbo_g1_dynamic}
    d\theta_t^{1, i_1}
    = -\lambda_1 \left( \theta_t^{1,i_1} - m_{\alpha}^{L_1} (\rho_t^N) \right) dt - \lambda_2 \nabla L_1(\theta_t^{1, i_1})\,dt +  \square \,dB_t^{1,i_1} 
    \quad \text{for }  i_1 \in [N_1],
\end{equation}
\begin{equation}\label{eq:fedcbo_g2_dynamic}
    d\theta_t^{2, i_2}
    = -\lambda_1 \left( \theta_t^{2,i_2} - m_\alpha^{L_2} (\rho_t^N) \right) dt - \lambda_2 \nabla L_2(\theta_t^{2, i_2})\,dt  +  \square \,dB_t^{2,i_2}
    \quad \text{for }  i_2 \in [N_2],
\end{equation}
\begin{equation}\label{eq:fedcbo_ConsensusPoint}
    m_{\alpha}^{L_1} (\rho_t^N) \propto \sum_{k=1,2} \sum_{i_k=1}^{N_k} \omega_{{L_1}}^{\alpha} (\theta_t^{k, i_k}) \theta_t^{k, i_k}, 
    \qquad
    m_{\alpha}^{L_2} (\rho_t^N) \propto \sum_{k=1,2} \sum_{i_k=1}^{N_k} \omega_{{L_2}}^{\alpha} (\theta_t^{k, i_k}) \theta_t^{k, i_k},
\end{equation}
\end{subequations}
with parameters $\lambda_1, \lambda_2, \alpha > 0$ and with weights $\omega_{L_j}^{\alpha} (\theta) := \exp(- \alpha L_j(\theta))$ for $j=1,2$.
The empirical measure of all particles is denoted by $\rho_t^N := \frac{N_1}{N} \rho_t^{N_1} + \frac{N_2}{N} \rho_t^{N_2}$, where $\rho_t^{N_1} := \frac{1}{N_1} \sum_{i_1=1}^{N_1} \delta_{\thetaCOne}$ and $\rho_t^{N_2} := \frac{1}{N_2} \sum_{i_2=1}^{N_2} \delta_{\thetaCTwo}$ represent the empirical measures of the particles in cluster $1$ and $2$, respectively.
Since the noise terms in the above dynamics will not be the main focus in the sequel, we abbreviate their coefficients with $\square$ for notational simplicity and refer the reader to \cite{carrillo2024fedcbo} for full details.

The term $m_{\alpha}^{L_k} (\rho_t^N)$ defined in \eqref{eq:fedcbo_ConsensusPoint}, which each agent is able to evaluate independently on their own respective loss function~$L_k$ without knowing cluster affiliations of other agents,
encodes the weighted average of the positions of \textit{all} $N$ particles $\{\theta_t^{1,i_1}\}_{i_1=1}^{N_1}$ and $\{\theta_t^{2,i_2}\}_{i_2=1}^{N_2}$ w.r.t.\@ the respective objective function~$L_k$.
By design, the consensus points $m_{\alpha}^{L_k} (\rho_t^N)$ will coincide within the clusters, thereby facilitating the automatic ``clustering'' of the agents without any knowledge of their cluster identities.
To demonstrate this mechanism, let us imagine for the moment that particles from cluster $1$ concentrate around the low-loss regions of $L_1$,
and presumably have smaller $L_1$ loss than particles from cluster $2$ which are expected to rather move around the low-loss regions of $L_2$, which are typically less favorable w.r.t.\@ $L_1$.
Therefore, in the computation of $m_{\alpha}^{L_1} (\rho_t^N)$, the particles $\{\theta_t^{1,i_1}\}_{i_1=1}^{N_1}$ from cluster $1$ are expected to receive higher weights compared to particles from cluster $2$, leading $m_{\alpha}^{L_1} (\rho_t^N)$ to approximate the weighted average of particles predominantly from cluster $1$.
An analogous rationale applies to $m_{\alpha}^{L_2} (\rho_t^N)$, which effectively implements an evolving weighted average that primarily includes particles from cluster $2$.
Thus, in the definitions of the consensus points in \eqref{eq:fedcbo_ConsensusPoint}, $L_1$ and $L_2$ act as ``selection criteria'' that effectively differentiate between agents from clusters $1$ and $2$, respectively.

In the dynamics described by \eqref{eq:fedcbo_g1_dynamic} and \eqref{eq:fedcbo_g2_dynamic}, respectively,
the first drift term can then be understood as the model exchange and local aggregation step, where agents first download the model parameters $\{\theta_t^{1,i_1}\}_{i_1=1}^{N_1}$ and $\{\theta_t^{2,i_2}\}_{i_2=1}^{N_2}$ from other agents and consecutively compute a weighted average thereof as previously described.
The second drift term (potentially together with an associated noise term) relates to the agent's local update step, where each agent runs (stochastic) gradient descent using only their own local datasets to update their model parameters in the absence of communication with other participants.

The FedCBO algorithm \cite{carrillo2024fedcbo} achieves great performance in the DCFL setting, provided that all agents aim to optimize their own objective functions.
In an adversarial scenario, however, i.e., as soon as malicious agents, which execute LF attacks as described in Section \ref{subsec:LF_attack}, are present in the system,
the FedCBO system becomes vulnerable and prone to undesired behavior, as we experimentally demonstrate in Section~\ref{sec:experiments}.
In particular, benign agents may struggle to distinguish between other benign agents and malicious agents within the same cluster (see Figures~\ref{subfig:fedcbo_avg_select_time} and \ref{subfig:fedcbo_avg_select_time_mul_weights}).
To intuitively understand the reason, let us for the purpose of demonstration assume that there are only three agents from cluster $1$ in the system; two benign agents $A$ and $B$, and one malicious agent $C$ who performs a LF attack on the source class $c_S$ with target class $c_T$.
Let us further suppose that agent $C$ has more resources such as local data samples than agents $A$ and $B$.
As discussed in Section \ref{subsec:LF_attack}, at each communication round the malicious agent $C$ performs a LF attack locally before sharing the poisoned model $\theta_t^C$ with agents $A$ and $B$.
\begin{figure}[htb]
\centering
\includegraphics[trim=215 200 275 200,clip,width=0.7\textwidth]{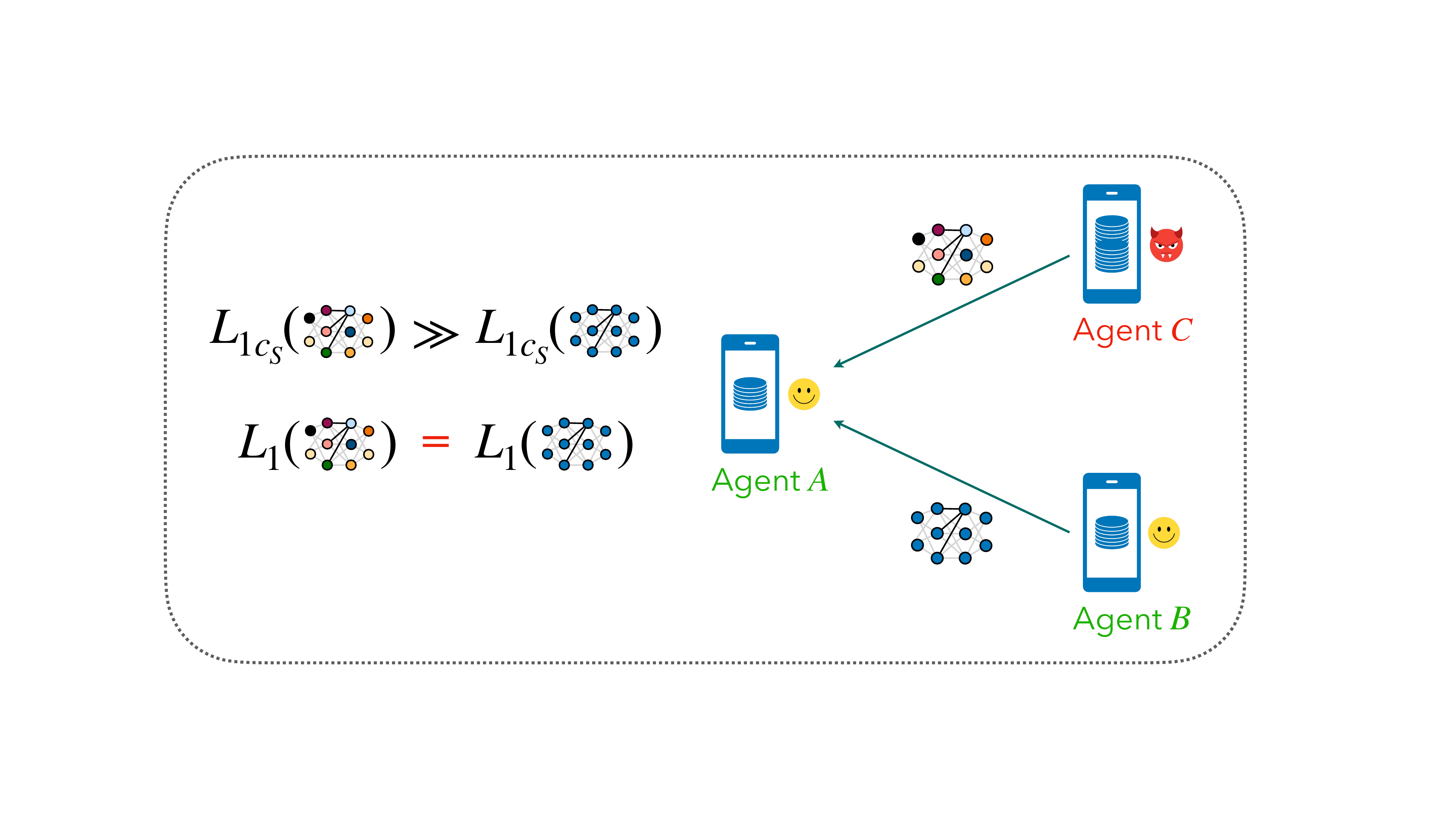}
\caption{An illustration of a successful LF attack. A malicious agent~$C$ with more resources may have similar $L_1$ loss as another benign agent~$B$, while performing a LF attack on source class $c_S$.
} 
\label{fig:fedcbo_fail}
\end{figure}
From the viewpoint of agent $A$ (see the associated illustration in Figure~\ref{fig:fedcbo_fail}),
during the local aggregation step, the following weights will be assigned to the models $\theta_t^B$ and $\theta_t^C$ based on \eqref{eq:fedcbo_ConsensusPoint}.
For $\theta_t^B$,
\begin{equation}
    \textstyle
    \omega_{L_1}^{\alpha}(\theta_t^B) =  \exp \left( -\alpha L_1 (\theta_t^B) \right) =  \exp \left(-\alpha \left( \sum_{c\neq c_S}
 w_{1,c} L_{1,c}(\theta_t^B)  + w_{1,c_S} L_{1,c_S} (\theta_t^B ) \right) \right),
\end{equation}
and, analogously, for $\theta_t^C$
\begin{equation}
    \textstyle
    \omega_{L_1}^{\alpha}(\theta_t^C) =  \exp \left( -\alpha L_1 (\theta_t^C) \right) = \exp \left(-\alpha \left( \sum_{c\neq c_S} 
 w_{1,c} L_{1,c}(\theta_t^C)  + w_{1,c_S} L_{1,c_S} (\theta_t^C ) \right) \right).
\end{equation}
Even though $L_{1, c_S}(\theta_t^C)\gg L_{1, c_S}(\theta_t^B)$ (as the malicious agent $\theta_t^C$ attacked the source class $c_S$),
the overall (average) loss $L_1(\theta_t^C)$ can still be equal, similar, or even smaller than the average loss $L_1(\theta_t^B)$, as the malicious agent $C$ has more resources and data points than the benign agent $B$,
which allows the attacker to achieve a substantially better performance for the poisoned model $\theta_t^C$ across all other classes $c \neq c_S$, i.e., $L_{1, c}(\theta_t^C) < L_{1, c}(\theta_t^B)$, compared to the benign model $\theta_t^B$. 
As a result, agent $A$ might not be able to distinguish the benign agent $B$ from the malicious agent $C$ by following the training protocol \eqref{eq:fedcbo_g1_dynamic}.
In other words, checking merely the local average losses $L_1$ or $L_2$ is insufficient to filter out malicious agents which perform LF attacks.

\subsection{\texorpdfstring{The FedCB\texorpdfstring{\textsuperscript{2}}{2}O System}{FedCB2O System}}\label{subsec:RobustFedCB2O}

The discussions in Sections \ref{sec:intro} and \ref{subsec:FailFedCBO} motivate to incorporate a suitable robustness criterion as a secondary layer~$G$ of evaluation for the benign agents in the DCFL optimization problem \eqref{eq:CFL_problem} to assess the trustworthiness of models from other agents.
Specifically, we consider to solve the bi-level optimization problems 
\begin{equation}\label{eq:DCFL_bilevel_opt}
    \thetaGk := \argmin_{\theta^{*, k}\in \Theta_k} G(\theta^{*, k})
    \quad \text{s.t.\@}\quad
    \theta^{*,k} \in \Theta_k := \argmin_{\theta \in \bbR^d} L_k(\theta) 
\end{equation}
for all $k \in [K]$ clusters ($K=2$ for simplicity) simultaneously, without violating the FL privacy protocol.
We again assume that agents from the same cluster $k$ share the lower-level objective function $L_k$ (see Remark \ref{rem:emp_loss} for a comment on the practical FL setting). 
Inspired by the FedCBO system \eqref{eq:fedcbo_dynamic} from \cite{trillos2024CB2O},
we extend the CB\textsuperscript{2}O dynamics~\eqref{eq:dyn_micro2} to the clustered setting, resulting in the FedCB\textsuperscript{2}O system, an interacting particle system describing the dynamics of the $N = N_1 + N_2$ particles (again, corresponding to model parameters of the agents) by
\begin{subequations}\label{eq:fedcb2o_dynamics}
\begin{equation}\label{eq:fedcb2o_g1_dynamics}
\begin{aligned}
    d\thetaCOne
    &= - \lambda_1 \left(\thetaCOne - m_{\alpha, \beta}^{G, L_1}(\rho_t^N) \right) dt - \lambda_2 \nabla L_1 (\thetaCOne) \,dt\\
    &\quad + \sigma_1 D \left(\thetaCOne - m_{\alpha, \beta}^{G, L_1}(\rho_t^N) \right) dB_t^{1,i_1} + \sigma_2 \Nbig{ \nabla L_1 (\thetaCOne)}_2 \,d\widetilde{B}_t^{1, i_1} 
    \quad \text{for }  i_1 \in [N_1],
\end{aligned}
\end{equation}
\begin{equation}\label{eq:fedcb2o_g2_dynamics}
\begin{aligned}
    \,d\thetaCTwo &= - \lambda_1 \left(\thetaCTwo - m_{\alpha, \beta}^{G, L_2}(\rho_t^N)\right) dt - \lambda_2 \nabla L_2 (\thetaCTwo) \,dt \\
    &\quad + \sigma_1 D \left(\thetaCTwo - m_{\alpha, \beta}^{G, L_2}(\rho_t^N)\right) dB_t^{2,i_2} + \sigma_2 \Nbig{ \nabla L_2 (\thetaCTwo) }_2 \,d\widetilde{B}_t^{2, i_2} 
    \quad \text{for }  i_2 \in [N_2].
\end{aligned}
\end{equation}
\end{subequations}
The empirical measure of all particles is denoted by $\rho_t^N$ and defined as in Section \ref{subsec:FailFedCBO}.
The consensus points $\mAlphaBetaLone{\rho_t^N}$ and $\mAlphaBetaLtwo{\rho_t^N}$ are now given as in \eqref{eq:consensus_point} replacing $L$ with $L_1$ and $L_2$, respectively, i.e., 
\begin{subequations}\label{eq:clustered_csp}
\begin{equation}\label{eq:clustered_csp_C1}
    m_{\alpha, \beta}^{G,L_1} (\rho_t^N) \propto \sum_{k=1,2}\, \sum_{ \theta_t^{k, i_k} \in Q_{\beta}^{L_1}[\rho_t^N]} \omegaa (\theta_t^{k, i_k}) \theta_t^{k, i_k},
\end{equation}
\begin{equation}\label{eq:cluster_csp_C2}
    m_{\alpha, \beta}^{G, L_2} (\rho_t^N) \propto \sum_{k=1,2}\, \sum_{ \theta_t^{k, i_k} \in Q_{\beta}^{L_2}[\rho_t^N]} \omegaa(\theta_t^{k, i_k}) \theta_t^{k, i_k},
\end{equation}
\end{subequations}
where $\omegaa(\theta) = \exp (-\alpha G (\theta))$, and where the sub-level sets $\QbetaLone{\dummy}$ and $\QbetaLtwo{\dummy}$ are defined as in \eqref{eq:Q_beta}, replacing $L$ with $L_1$ and $L_2$, respectively.
Notice that each agent is again able to evaluate the consensus point independently on their own loss function without knowledge of cluster membership.
Analogously to FedCBO in Section~\ref{subsec:FailFedCBO},
the FedCB\textsuperscript{2}O dynamics~\eqref{eq:fedcb2o_dynamics} has two key features.
The first, corresponding to the model exchange and local aggregation step in the DCFL paradigm, is the computation of the consensus point as of \eqref{eq:clustered_csp}.
The consensus point $\mAlphaBetaLone{\rho_t^N}$ is now computed as a weighted (w.r.t.\@ the robustness criterion $G$) average of those particles from both $\{\thetaCOne\}_{i_1=1}^{N_1}$ and $\{\thetaCTwo\}_{i_2=1}^{N_2}$ that belong to the sub-level set $\QbetaLone{\rho_t^N}$,
which, analogously to the sub-level set \eqref{eq:Q_beta} defined in the CB\textsuperscript{2}O system, can be regarded as an approximation of the neighborhood of the set $\Theta_1$ of all global minimizers of $L_1$.
Since particles from cluster $2$ are more likely to concentrate around $\Theta_2$, which are typically suboptimal w.r.t.\@ $L_1$, they don't affect the location of the consensus point.
Therefore, $\mAlphaBetaLone{\rho_t^N}$ predominantly incorporates contributions from particles from cluster $1$, which have small $L_1$ loss.
In other words, the sub-level set $\QbetaLone{\rho_t^N}$ acts as a first filter that excludes particles not belonging to the same cluster or particles with bad loss. 
The weighted averaging based on the robustness criterion $G$, on the other hand, serves as a second level of filtering,
that mitigates the influence of poisoned models with small $L_1$ losses but embedded misclassification biases, such as those obtained by malicious agents performing LF attacks. 
The second key feature in \eqref{eq:fedcb2o_dynamics} are the local gradient terms $\nabla L_1$ and $\nabla L_2$, which have the same interpretation as in FedCBO system \eqref{eq:fedcbo_dynamic}. 
They correspond to the agents' local update steps, where each agent runs (stochastic) gradient descent using their own datasets to update their local models in the absence of communication with others.

\begin{remark}
    In the FedCB\textsuperscript{2}O system~\eqref{eq:fedcb2o_dynamics},
    we include only the dynamics of the benign agents from the different clusters for notational simplicity.
    The presence of malicious agents in the FedCB\textsuperscript{2}O system can be incorporated similarly as in \eqref{eq:dyn_macro}, with their influence reflected in the computation of the consensus points.
    More precisely, there are then $N=N^b_1+N^b_2+N^m$ particles present in the system with the empirical measure given as $\rho_t^N := \frac{N^b_1}{N} \rho_t^{N^b_1} + \frac{N^b_2}{N} \rho_t^{N^b_2} + \frac{N^m}{N} \rho_t^{N^m}$.
\end{remark}

\begin{remark}[Theoretical analysis of FedCB\textsuperscript{2}O] \label{rem:DCFL_FedCB2O_theory}
    The mean-field convergence statements \cite[Theorem~2.7]{trillos2024CB2O} as well as Theorem~\ref{thm:main} for CB\textsuperscript{2}O in both an attack-free as well as adversarial setting,
    can be extended to the FedCB\textsuperscript{2}O dynamics~\eqref{eq:fedcb2o_dynamics}
    by leveraging the theoretical contributions of \cite{carrillo2024fedcbo}, where FedCBO has been analyzed using the analytical framework of CBO~\cite{fornasier2021consensus,riedl2022leveraging}.
    
    This permits to prove convergence in mean-field law for the FedCB\textsuperscript{2}O system~\eqref{eq:fedcb2o_dynamics} to the global minimizers~$\thetaGk$ of the bi-level optimization problems \eqref{eq:DCFL_bilevel_opt} by establishing exponentially fast decay of $\sum_{k=1}^2W_2^2(\rho^{b,k}_t,\delta_{\thetaGk})$, where $\rho^{b,1}$ and $\rho^{b,2}$ denote the laws of the corresponding mean-field limits of \eqref{eq:fedcb2o_g1_dynamics} and \eqref{eq:fedcb2o_g2_dynamics}, respectively.
\end{remark}

\begin{remark} \label{rem:emp_loss}
    In practical real-world FL and ML settings in general,
    each agent has access to only a finite number of data samples.
    This results in empirical loss functions that typically differ slightly even for agents in the same cluster.
    Our modeling assumption, where all agents in cluster $k$ share the same lower-level objective function $L_k$ for $k = 1,2$ is therefore a simplification. 
    To better reflect a realistic setting at the algorithmic level, we assume in the subsequent Sections \ref{subsec:RobustFedCB2O} and \ref{sec:experiments} that each agent $j$ possesses a lower-level objective $\AgentJObj$, which can be viewed as a slight perturbation of the underlying ``true" loss function $L_k$, depending on the cluster agent $j$ belongs to.
    The computation of the consensus point \eqref{eq:clustered_csp} for agent $j$ remains the same, except that $L_k$ is replaced with $\AgentJObj$, leading to slightly different consensus points, even for agents in the same cluster.
    We leave the mathematical modeling and analysis of this more realistic scenario for future work.
\end{remark}

\subsection{\texorpdfstring{The FedCB\texorpdfstring{\textsuperscript{2}}{2}O Algorithm}{The FedCB2O Algorithm}}\label{subsec:FedCB2O_alg}

To transform the interacting multi-particle system \eqref{eq:fedcb2o_dynamics} into an algorithm that is practicable in real-world FL problems,
a series of adjustments are required.
First of all, we follow the discretization proposed for the FedCBO algorithm in \cite{carrillo2024fedcbo}.
This involves the following three steps:
(i) We discretize the continuous-time dynamics \eqref{eq:fedcb2o_dynamics} using an Euler-Maruyama scheme;
(ii) The gradient term and the noise terms in \eqref{eq:fedcb2o_dynamics} are replaced with mini-batch stochastic gradient descent (SGD);
(iii) We apply the splitting scheme proposed in \cite{carrillo2024fedcbo}, where $\tau$ steps of local SGD are performed, followed by a single model exchange and local aggregation step which corresponds to the consensus point computation.
Applying this discretization scheme to the FedCB\textsuperscript{2}O system \eqref{eq:fedcb2o_dynamics} yields the general framework of the FedCB\textsuperscript{2}O algorithm, which we detail in Algorithm~\ref{alg:FedCB2O}.
\begin{algorithm}[!htb]
\setstretch{1.25}
\caption{FedCB\textsuperscript{2}O (benign agents)}
\label{alg:FedCB2O}
\begin{algorithmic}[1]
\REQUIRE
Number of iterations $T$; number of local gradient steps $\tau$; model download budget $M$; CB\textsuperscript{2}O  hyperparameters $\lambda_1, \lambda_2, \alpha, \beta$; discretization time step size $\gamma$; initialized sampling likelihood $P_0 \in \R^{N \times (N-1)}$ as the zero matrix;
\STATE Initialize models $\theta_0^j \in \R^d$ for $j \in [N]$
\FOR{$n=0, \dots, T-1$}
\STATE \textbf{LocalUpdate} ($\theta_n^j, \tau, \lambda_2, \gamma$) for $j \in [N]$;\\
\STATE \textbf{LocalAggregation} (agent $j$) for $j \in [N]$;
\ENDFOR
\RETURN $\theta_T^j$ for $j \in [N]$.
\end{algorithmic}

\underline{\textbf{LocalUpdate ($\widehat{\theta}_0, \tau, \lambda_2, \gamma$)} for agent $j$:} 
\begin{algorithmic}[1]
\FOR{$q=0, \dots, \tau - 1$}
\STATE (stochastic) gradient descent $\widehat{\theta}_{q+1} \leftarrow \widehat{\theta}_q - \lambda_2 \gamma \nabla \widetilde{L}_j(\widehat{\theta}_q)$;
\ENDFOR
\RETURN $\widehat{\theta}_{\tau}$.
\end{algorithmic}
\end{algorithm}

Like conventional FL algorithms, our FedCB\textsuperscript{2}O algorithm consists of two main steps per communication round: a local update step, and a model exchange and local aggregation step.
As discussed before, during the local update step each agent performs $\tau$ steps of SGD independently on its own local device.
Let us now focus on the model exchange and local aggregation step and discuss how the consensus point computation is modified to accommodate practical demands.
For clarity, we take the viewpoint of agent $j$, who is one of the participants in the FedCB\textsuperscript{2}O system.

In the original consensus point computation \eqref{eq:clustered_csp}, agent $j$ is required to download the models of all other $N-1$ participants before performing their evaluation w.r.t.\@ $\AgentJObj$ and $G$ and the weighted averaging.
However, in practice, the total number of participants $N$ in the system is typically large, especially in ``cross-device'' FL settings \cite{konevcny2016federatedopt,karimireddy2021breaking, yang2022practical, chen2023fs, karagulyan2024spam}.
Due to communication and storage limitations on local devices, it is therefore infeasible for agent $j$ to download all other models.
To take this practical constraint into consideration, let us therefore assume w.l.o.g.\@ that agent $j$ has a model download budget of $M \ll N$ models per communication round. 
This means that agent $j$ can only select $M$ models from the $N-1$ other participants to download per communication round, which we refer to as the agent selection process.
This rationale can be also transferred to the situation where some agents might be offline during the time agent $j$ is in the model aggregation phase.

The most straightforward approach is to uniformly sample $M$ agents at random, a strategy commonly used in conventional DFL algorithms \cite{sun2022decentralized} as well as IPS-based optimization \cite{carrillo2019consensus,jin2020random},
where it is known under the name random batch method.
However, this strategy may be inefficient, as agent $j$ ideally seeks to maximize the utility of its model download budget by selecting the most promising models for the subsequent aggregation step in order to accelerate training.
Notably, the computation of the consensus point \eqref{eq:clustered_csp} inherently involves already an agent selection process as it is computed as a weighted average over the \textit{subset} of those positions (model parameters) that have small loss values on agent $j$'s lower-level objective function $\AgentJObj$.
Yet, determining which models to include in the averaging would require agent $j$ to first download all models and evaluate them on its local dataset, a task that we wanted to avoid.
To resolve this issue, we propose that agent $j$ maintains a historical record of other participants in the system from whom it has downloaded models in previous communication rounds.
Specifically, agent $j$ stores a vector $P_n^j := \big(P_n^{j,1}, \dots, P_n^{j, j-1}, P_n^{j, j+1}, \dots, P_n^{j, N} \big) \in \R^{N-1}$, which encodes the potential ``benefit'' of choosing another agent's model, based on their past performances on agent $j$'s local dataset, i.e., the lower-level objective $\AgentJObj$, up to the communication round $n$.
In Remark \ref{rem:sampling_likelihood_idea} we explain the relationship between $P_n^j$ and the original particle selection principle present in the consensus point computation \eqref{eq:clustered_csp}.
During the agent selection process in round $n+1$, agent $j$ uses $P_n^j$ as a sampling likelihood to select $M$ models from the available participants.
We refer to Remark \ref{rem:sampling_likelihood_init} and the ProbSampling method in Algorithm~\ref{alg:localAggregation} for implementational details regarding the sampling likelihood $P_n^j$ and the agent selection process.

By incorporating the agent selection strategy discussed above,
we are now ready to present the local aggregation step in Algorithm \ref{alg:localAggregation}.
At the $(n+1)$-th iteration of the FedCB\textsuperscript{2}O algorithm,
after downloading $M$ models selected using the ProbSampling method, agent $j$ updates the sampling likelihood $P_{n}^j$ for the corresponding selected $M$ agents.
This update is based on the recent performance of their models on agent $j$'s loss function $\AgentJObj$, combined with their historical records through an exponential moving average controlled by the hyperparameter $\zeta$, as described in \eqref{eq:update_sampling_likelihood}.
Next, agent $j$ evaluates these $M$ models using the robustness criterion $G$, computes the consensus point as defined in \eqref{eq:comp_csp}, and updates its local model according to \eqref{eq:local_agg_update}.
This completes one round of the local aggregation step.
\begin{algorithm}[!htb]
\setstretch{1.25}
\caption{LocalAggregation (benign agent $j$)}
\label{alg:localAggregation}
\begin{algorithmic}[1]
\REQUIRE Agent $j$'s model $\theta_n^j \in \R^d$;  
sampling likelihood $P_n^j \in \R^{N-1}$; 
CB\textsuperscript{2}O hyperparameters $\lambda_1, \alpha, \Temp $; 
step size $\gamma$; 
model download budget $M$;
moving average parameter $\zeta$; \\
\STATE Set $A_n \leftarrow$ \textbf{ProbSampling} ($P_n^j, M$);\\
\STATE Download models $\theta_n^i$ for $i \in A_n$;\\
\STATE Evaluate models $\theta_n^i$ on agent $j$'s dataset and denote the corresponding losses as $\widetilde{L}_j^i$, $i \in A_n$;\\
\STATE Update sampling likelihood $P_n^j$ as\vspace{-5pt}
\begin{equation}\label{eq:update_sampling_likelihood}
P_{n+1}^{j, i} \leftarrow (1 - \zeta) P_{n}^{j, i} + \zeta \exp \!\big(\! -\Temp \widetilde{L}_j^i \big), \qquad \text{for} \;\; i \in A_n;
\end{equation}
\vspace{-15pt}
\STATE Evaluate models $\theta_n^i$ on the robustness criterion $G$ and denote their values by $G^i$, $i \in A_n$;\\
\STATE Compute consensus point $m_j$ by
\begin{equation}\label{eq:comp_csp}
m_j \leftarrow \frac{1}{\sum_{i \in A_n} \mu_j^i} \sum_{i \in A_n} \theta_n^i \mu_j^i \qquad \text{with } \mu_j^i = \exp(-\alpha G^i);
\end{equation}
\vspace{-10pt}
\STATE Update agent $j$'s model by
\begin{equation}\label{eq:local_agg_update}
\theta_{n+1}^j \leftarrow \theta_n^j - \lambda_1 \gamma \big(\theta_n^j - m_j \big);
\end{equation}
\vspace{-15pt}
\RETURN $\theta_{n+1}^j, P_{n+1}^j$.
\end{algorithmic}
\underline{\textbf{ProbSampling} 
($P_n^j, M$):}
\begin{algorithmic}[1]
\IF{ $ S := \{i \in [N] : P_n^{j,i} = 0 \} \neq \emptyset$}
\STATE For $A_n$, randomly pick $M$ agents in set $S$ uniformly if $|S| > M$ else pick set $S$;
\ELSE 
\STATE For $A_n$, randomly pick $M$ agents among set $[N] \backslash j$ with probability (normalized) $P_n^j$;
\ENDIF
\RETURN $A_n$.
\end{algorithmic}
\end{algorithm}

\begin{remark}\label{rem:sampling_likelihood_idea}
    In essence,
    the design of the sampling likelihood $P_n^j$ adapts the particle selection principle used in the computation of the consensus point \eqref{eq:clustered_csp},
    replacing the deterministic criterion based on the current loss function with a probabilistic approach based on the historical performance records encoded in $P_n^j$.
    The temperature hyperparameter $\Temp$ used in \eqref{eq:update_sampling_likelihood} plays a role analogous to the quantile hyperparameter $\beta$ in the sub-level set $Q^{\AgentJObj}_{\beta}[\dummy]$ as defined in \eqref{eq:Q_beta}.
    Specifically, $\Temp$ controls the likelihood that agent $j$ will prioritize selecting models with small $\AgentJObj$ values in subsequent communication rounds, similar to how $\beta$ determines the number of particles included in the consensus point computation based on their $\AgentJObj$ values. 
    This modification facilitates the practical feasibility of the consensus point computation \eqref{eq:clustered_csp} in real-world FL settings while preserving its core principles.
\end{remark}

\begin{remark}[Initialization of the sampling likelihood $P_0^j$]\label{rem:sampling_likelihood_init}
    At the beginning of the FedCB\textsuperscript{2}O algorithm, agent $j$ has no information about the other participants in the system.
    Consequently, the sampling likelihood $P_0^j \in \R^{N-1}$ is initialized as a zero vector.
    Until agent $j$ has selected all other agents at least once in previous communication rounds, it prioritizes selecting agents that have not yet been chosen.
    This ensures that agent $j$ has a preliminary assessment of the usefulness of all other agents' models before beginning the agent selection process based on their historical performance.
    The agent selection strategy is detailed as ProbSampling in Algorithm \ref{alg:localAggregation} and may be of independent interest to any DFL algorithm that requires an agent selection mechanism \cite{chen2022optimal,fu2023client}.
\end{remark}


\subsection{Experiments}\label{sec:experiments}

Let us first describe in detail our experimental setup.

\textbf{Dataset \& Attack Setup.}
We adopt the DCFL setting from \cite{carrillo2024fedcbo} using the standard EMNIST dataset \cite{cohen2017emnist},
which contains a total of $47$ classes comprising $10$ digits and $37$ letters in the English alphabet (lower and upper case).
We introduce malicious agents into the system to evaluate our method's robustness.

We begin by choosing a subset of $35,500$ training samples and $18,800$ test samples from the original EMNIST dataset.
To create clusters, we augment the dataset by applying rotations of $0^{\circ}$ and $180^{\circ}$ to each image, producing in this way $K=2$ clusters, each corresponding to one of the two rotation angles.
Each cluster contains $35,500$ training samples and $18,800$ test samples.
The system is configured with $N=100$ agents, evenly split across the two clusters.
Within each cluster, we assign $35$ benign agents and $15$ malicious agents.
For training, the $35,500$ training images in each cluster are randomly partitioned such that each benign agent is assigned $500$ images, while each malicious agent is assigned $1,200$ images, enhancing the overall capability of malicious agents compared to benign ones. 
The test data is the same for all agents in the same cluster. 
Examples of points in the rotated EMNIST dataset are depicted in Figure \ref{fig:rotated_emnist}.

Malicious agents attempt to attack benign agents of the same cluster by executing LF attacks, as described in Section~\ref{subsec:LF_attack}, targeting the source class images of ``O'' (upper case letter O) and relabeling them as ``$0$'' (number zero).
\begin{figure}[!htb]
\centering
\includegraphics[trim=20 20 20 20,clip,width=0.6\textwidth]{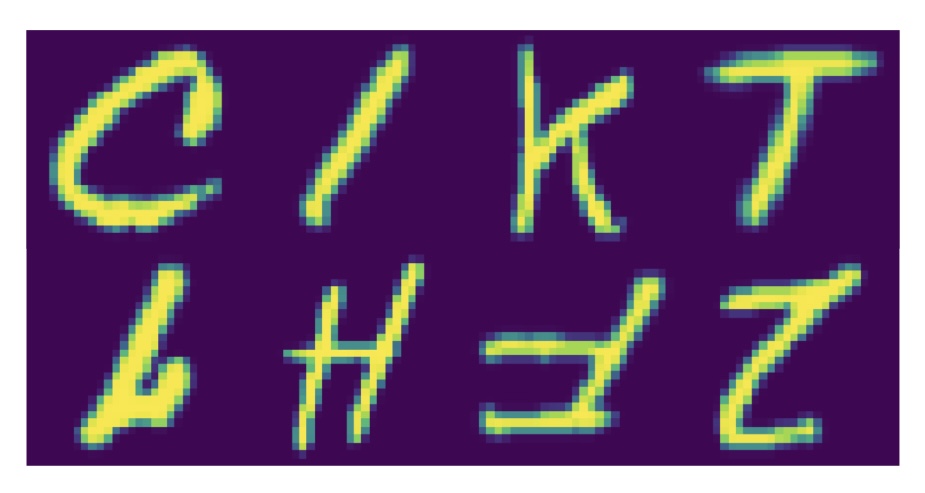}
\caption{Samples of the rotated EMNIST dataset.
Each row contains samples from one rotation.} 
\label{fig:rotated_emnist}
\end{figure}

\textbf{Baselines \& Implementations.}
We compare our FedCB\textsuperscript{2}O algorithm (Algorithm \ref{alg:FedCB2O}) with two baselines: FedCBO \cite{carrillo2024fedcbo}, and DFedAvgM \cite{sun2022decentralized} under both the attack-free and non-clustered settings, which we refer to as {Oracle}.
As a base model, we use a neural network (NN) with two convolutional layers, two max-pooling layers, followed by two dense layers with ReLU activations.
The total number of communication rounds is set to $T = 150$, with all agents participating in the training in every round.
During the local update step, each agent performs $\tau = 5$ epochs of mini-batch SGD with batch size of $64$, learning rate $\gamma = 0.004$, and momentum $0.9$.
The model download budget for each agent is set to $M = 20$.
To strengthen the attack, we assume that the malicious agents have knowledge of the clustering structure (which remains unknown to the benign agents) as well as the identities of other attackers in the system.
During the agent selection process, each malicious agent prioritizes selecting models from other attackers within the same cluster.
Once these models are selected, the remaining download budget is used to randomly pick models from benign agents within the same cluster.
Moreover, since malicious agents already know the identities of other participants, applying the aggregation method is not required for them.
Instead, during the local aggregation step, malicious agents perform weighted averaging based on the number of data points used to train the downloaded models, a commonly used approach in FL algorithms \cite{mcmahan2017communication, sun2022decentralized}.
We now provide additional implementation details for each baseline algorithm.
\begin{itemize}
    \item \textbf{FedCB\textsuperscript{2}O.}
    Since the FedCB\textsuperscript{2}O algorithm requires benign agents to evaluate downloaded models from other agents on their own local datasets, we partition the datasets of all benign agents accordingly.
    Each benign agent's dataset of $500$ samples is randomly split into $400$ training and $100$ validation samples. 
    The training set is used for local updates, whereas the validation set is used to evaluate downloaded models.
    We set the CB\textsuperscript{2}O hyperparameters to $\lambda_1 = 10$, $\lambda_2 = 1$ and $\alpha = 10$.
    In the {ProbSampling} method (Algorithm \ref{alg:localAggregation}), we use a temperature parameter of $\Temp = 2$ and a moving average parameter of $\zeta = 0.5$.
    The design of the robustness criterion $G$ is described in Remark \ref{rem:func_G}.

    \item \textbf{FedCBO.}
    We set the hyperparameters in the FedCBO algorithm \cite{carrillo2024fedcbo} to match those of FedCB\textsuperscript{2}O, except for $\alpha = 10$, which we empirically found to yield better performance for FedCBO. 
    For a fair comparison, we replace the {$\varepsilon$-greedy sampling} agent selection strategy originally used in FedCBO \cite{carrillo2024fedcbo} with the {ProbSampling} method proposed in this paper (Algorithm \ref{alg:localAggregation}).
    We also use the same train validation split for all benign agents as in the FedCB\textsuperscript{2}O implementation.

    \item \textbf{Oracle.}
    DFedAvgM \cite{sun2022decentralized} is a commonly used algorithm for DFL when data is homogeneous and when there are no attacks on the system.
    For a fair comparison, we run DFedAvgM in a setting where there is only one cluster (no rotations), so that the algorithm doesn't need to deal with the clustering structure.
    Additionally, we make the setup for DFedAvgM attack-free by considering two cases: (i) Removing the $15$ ``malicious'' agents entirely from the system; (ii) Keeping the ``malicious'' agents in the system but without them performing attacks.
    Both of these settings are attack-free and represent ideal scenarios where DFedAvgM is expected to perform at its best.
    For this reason, we use these scenarios as benchmarks that indicate the best performance one can hope for given our experimental setup.
    In the case (i), removing the ``malicious'' agents means losing some of the correct training data they would have contributed. 
    This is expected to slightly lower the overall accuracy of the final training results, so we refer to this setup as \textbf{Oracle Min}.
    In the case (ii), the ``malicious'' agents are kept in the system without carrying out any attacks. 
    This allows the system to benefit from a larger total number of data samples while remaining attack-free, which is expected to yield the best performance. 
    We refer to this setup as \textbf{Oracle Max}.
\end{itemize}

\begin{remark}[Robustness criterion $G$]\label{rem:func_G}
    As discussed in Section \ref{sec:intro}, defending against different attacks requires to select an appropriate robustness criterion $G$ in the bi-level optimization problem \eqref{eq:DCFL_bilevel_opt}.
    In the DCFL setting with malicious agents performing LF attacks, we design the upper-level objective function for a benign agent $j$ as
    \begin{equation}\label{eq:robust_criterion_G_LF}
        G_j (\theta; \theta^j) := \max_{c \in [C]}\, \widetilde{L}_{j, c} (\theta) - \widetilde{L}_{j, c} (\theta^j),
    \end{equation}
    where $\theta$ represents the model downloaded and evaluated by agent $j$, and $\theta^j$ denotes the current model of agent $j$. 
    $\widetilde{L}_{j, c}$ is the loss for class $c \in [C]$ given agent $j$'s local dataset.
    The intuition behind the robustness criterion \eqref{eq:robust_criterion_G_LF} is as follows. Agent $j$ determines whether the model $\theta$ is poisoned by evaluating its similarity to the agent's own model $\theta^j$ across all classes of the locally stored dataset.
    If there exists at least one class where $\theta$ and $\theta^j$ exhibit significantly different performance, $\theta$ is treated as a poisoned model and assigned a low weight during the averaging process \eqref{eq:clustered_csp}.

    The robustness criterion $G_j$, as defined in \eqref{eq:DCFL_bilevel_opt}, is ``personalized'' to agent $j$ and leverages the agent's own model. 
    This differs slightly from the original bi-level optimization framework \eqref{eq:DCFL_bilevel_opt}, where the robustness criterion $G$ doesn't depend on the agent's own changing parameter. This more complex and personalized robustness criterion $G_j$ demonstrates promising empirical performance.
    We leave the theoretical analysis of frameworks incorporating upper-level objectives similar to \eqref{eq:robust_criterion_G_LF} for future work.
\end{remark}

\textbf{Performance Metrics.}
We evaluate the performance of the different algorithms by computing the average predicition accuracies of benign agents' models on the test data that shares the same distribution as their training data (i.e., data points with the same rotation).
The evaluation is based on the following three metrics: (i) The average models' prediction accuracy\footnote{\revisedOne{The prediction accuracy of a model is computed as  $\frac{\text{Number of correctly predicted data points}}{\text{Total number of data points}} \times 100 \%$.}} across all classes (abbreviated as overall acc);
(ii) The average models' prediction accuracy on the source class (abbreviated as source class acc);
(iii) The probability that benign agents' models predict the source class samples as the target class label, referred to as the attack success rate (ASR).
Source class accuracy and ASR specifically measure the robustness of an algorithm against LF attacks.
All experiments are conducted using $3$ different random seeds for each algorithm, and the reported results represent the averages across these runs.

\textbf{Experimental Results.} The test results are summarized in Table \ref{tab:baselines}, and we make the following observations.
\begin{itemize}
    \item[(i)] The baseline Oracle Max achieves the best (or second-best) performance across all three metrics.
    This is expected, as this baseline represents an idealized scenario where the data of all agents is included and the system is attack-free.
    The Oracle Min baseline, which excludes the malicious agents from the system, achieves similar accuracy on the source class and attack success rate (ASR) as Oracle Max, but exhibits lower overall accuracy.
    This difference is expected since Oracle Max benefits from the additional data samples contributed by the ``malicious'' agents (who are not attacking in this setting), thereby providing more information to the system.

    \item[(ii)] FedCBO achieves an overall accuracy comparable to the baseline Oracle Max.
    This demonstrates that the FedCBO algorithm enables agents to implicitly identify the cluster identities of other participants during training, without prior knowledge of these identities.
    The high overall accuracy of FedCBO, despite the presence of malicious agents performing label-flipping (LF) attacks, is due to the fact that these attacks target only one specific class, resulting in minimal impact on the overall accuracy.
    This highlights why LF attacks are generally difficult to detect.
    However, FedCBO shows a significant degradation in accuracy on the source class, coupled with a notably high ASR.
    This indicates that FedCBO fails to effectively filter out malicious agents during the training process.
    Further evidence is shown in Figures \ref{subfig:fedcbo_avg_select_time} and \ref{subfig:fedcbo_avg_select_time_mul_weights}, which reveal that, during training, benign agents select models from malicious agents at a frequency comparable to their selection of models from other benign agents within the same cluster.
    Moreover, they assign nearly equal weights to malicious agents in the weighted averaging process,
    highlighting that the lower-level objective $\AgentJObj$ fails to distinguish between benign and malicious agents within the same cluster.
    
    \item[(iii)] Compared to FedCBO, our FedCB\textsuperscript{2}O algorithm achieves significantly higher source class accuracy and a much lower ASR, both of which are comparable to the idealized baselines Oracle Min and Oracle Max.
    Furthermore, the overall accuracy of FedCB\textsuperscript{2}O is very close to the one of Oracle Min, where malicious agents are completely removed from the system.
    These results demonstrate that FedCB\textsuperscript{2}O not only retains the ability of benign agents to distinguish the cluster identities of other agents as the FedCBO algorithm does, but also effectively mitigates the influence of malicious agents, ensuring robustness against LF attacks.
    For a further empirical illustration of the importance of incorporating the robustness criterion in the FedCB\textsuperscript{2}O algorithm, please refer to Remark \ref{rem:importance_G} and Figures \ref{subfig:fedcb2o_avg_select_time} and \ref{subfig:fedcb2o_avg_select_time_mul_weights}.
\end{itemize}

\begin{table}[!htb]
\centering
\caption{Comparison of FedCBO and FedCB\textsuperscript{2}O with baselines (DFedAvgM) across three performance metrics (in $\%$) with standard deviations.}
\label{tab:baselines}
\renewcommand\arraystretch{1.25}
    \begin{sc}
\begin{tabular}{ccccc}
\toprule
    \bf   & \bf FedCBO & \bf Oracle Min & \bf Oracle Max & \bf FedCB\textsuperscript{2}O \\
\midrule
Overall Acc  & $84.26 \pm 0.18 $  & $82.75 \pm 0.24$ & $\bf{84.43 \pm 0.09}$ & $82.79 \pm 0.14$ \\
\midrule
Source Class Acc & $40.23 \pm 4.02 $ & $63.53 \pm 1.97$ & $\bf{63.80 \pm 1.76}$ & $\bf{55.73 \pm 2.94}$ \\
\midrule
Attack Success Rate & $55.23 \pm 3.81 $ & $\bf{31.08 \pm 2.64}$ & $32.19 \pm 1.93$ & $\bf{38.73 \pm 3.42}$ \\
\bottomrule
\end{tabular}
    \end{sc}
\end{table}

\begin{remark}[Improved FedCB\textsuperscript{2}O algorithm]
    \label{rem:FedCB2Oimproved}
    As visible from Table \ref{tab:baselines}, the performance of FedCB\textsuperscript{2}O is similar to Oracle Min w.r.t.\@ all three metrics.
    This demonstrates that the FedCB\textsuperscript{2}O algorithm effectively excludes malicious agents by taking weighted averages based on the robustness criterion~$G$ designed in Remark~\ref{rem:func_G}.
    However, since the malicious agents target only one specific class, the data they possess for other classes may still be valuable to benign agents.
    Ideally, the algorithm should maximize the use of this useful information while minimizing the impact of the attacks. 
    To achieve this, we introduce a hyperparameter $T_G$, which determines the communication round at which benign agents begin leveraging the robustness criterion to eliminate the influence of malicious agents from the averages.
    Specifically, during the initial stages of training (i.e., before round $T_G$), benign agents simply use the average losses based on their own local datasets as the weighting criterion, similar to the approach in FedCBO \cite{carrillo2024fedcbo}.
    This allows benign agents to exploit the valuable information provided by malicious agents early in the training process.
    Starting at communication round $T_G$, benign agents then switch to using the robustness criterion~$G$ (as defined in Remark \ref{rem:func_G}) to eliminate the contributions of malicious agents, ensuring that benign agents are not significantly affected by the attacks in the long term.
    
    The results of the FedCB\textsuperscript{2}O algorithm, incorporating the hyperparameter $T_G$, are summarized in Table~\ref{tab:FedCB2O_T_G}.
    These results demonstrate that utilizing useful information from malicious agents during the early stages of training improves the overall accuracy of the algorithm.
    At the same time, activating the robustness criterion at the appropriate communication round $T_G$ ensures that the performance of benign agents on the source class does not degrade.
    In our experiments, a value of $T_G = 30$ strikes the best balance between leveraging useful information from malicious agents and effectively defending against their attacks.
\end{remark}

\setlength{\abovecaptionskip}{5pt}
\setlength{\belowcaptionskip}{0pt}
\begin{figure}[!htb]
\centering
	\subcaptionbox{FedCBO Algorithm\label{subfig:fedcbo_avg_select_time}}{
		\includegraphics[ width=0.46\textwidth]{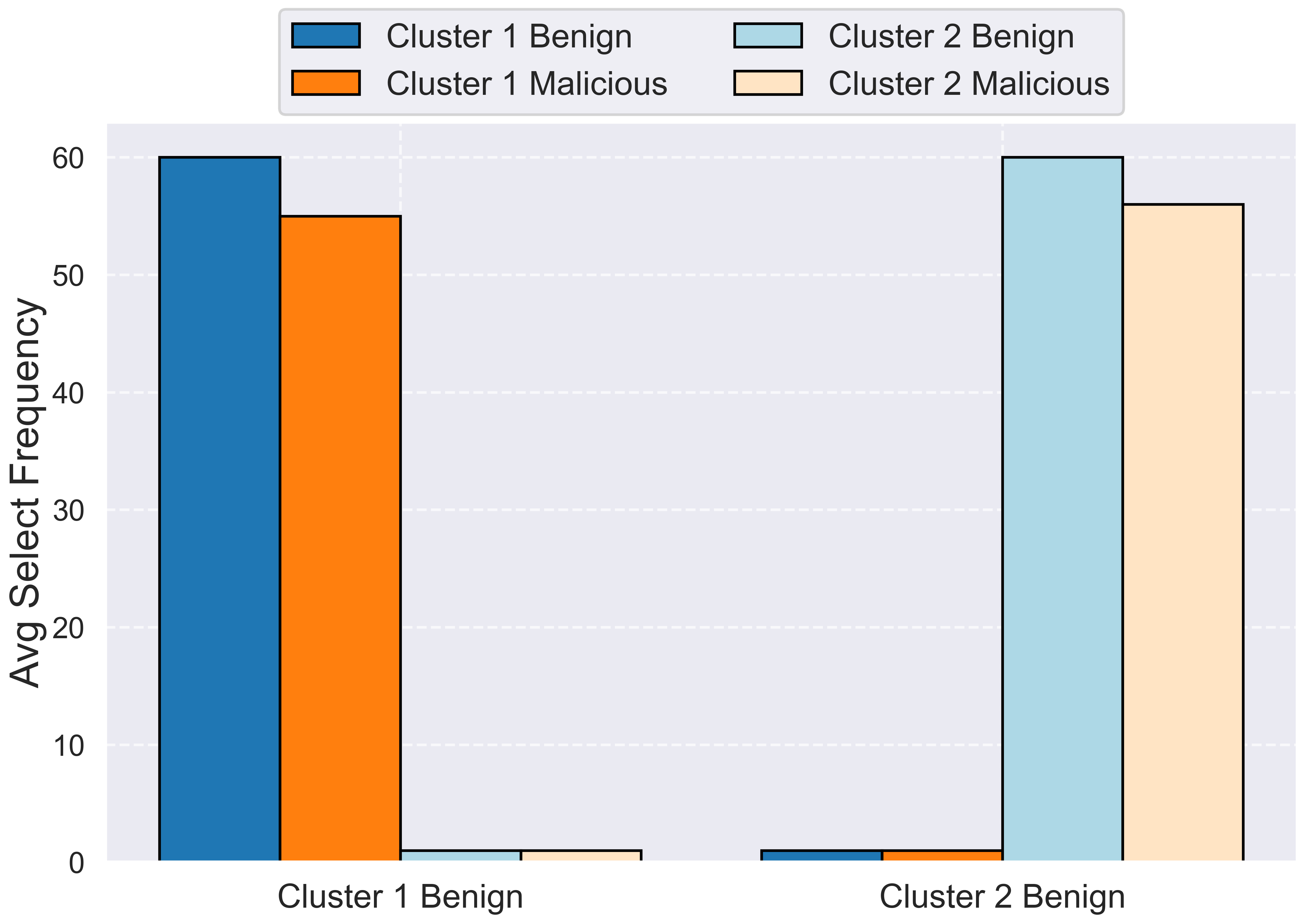}
	}
    \hspace{2em}
	\subcaptionbox{FedCBO Algorithm\label{subfig:fedcbo_avg_select_time_mul_weights}}{
\includegraphics[width=0.46\textwidth]{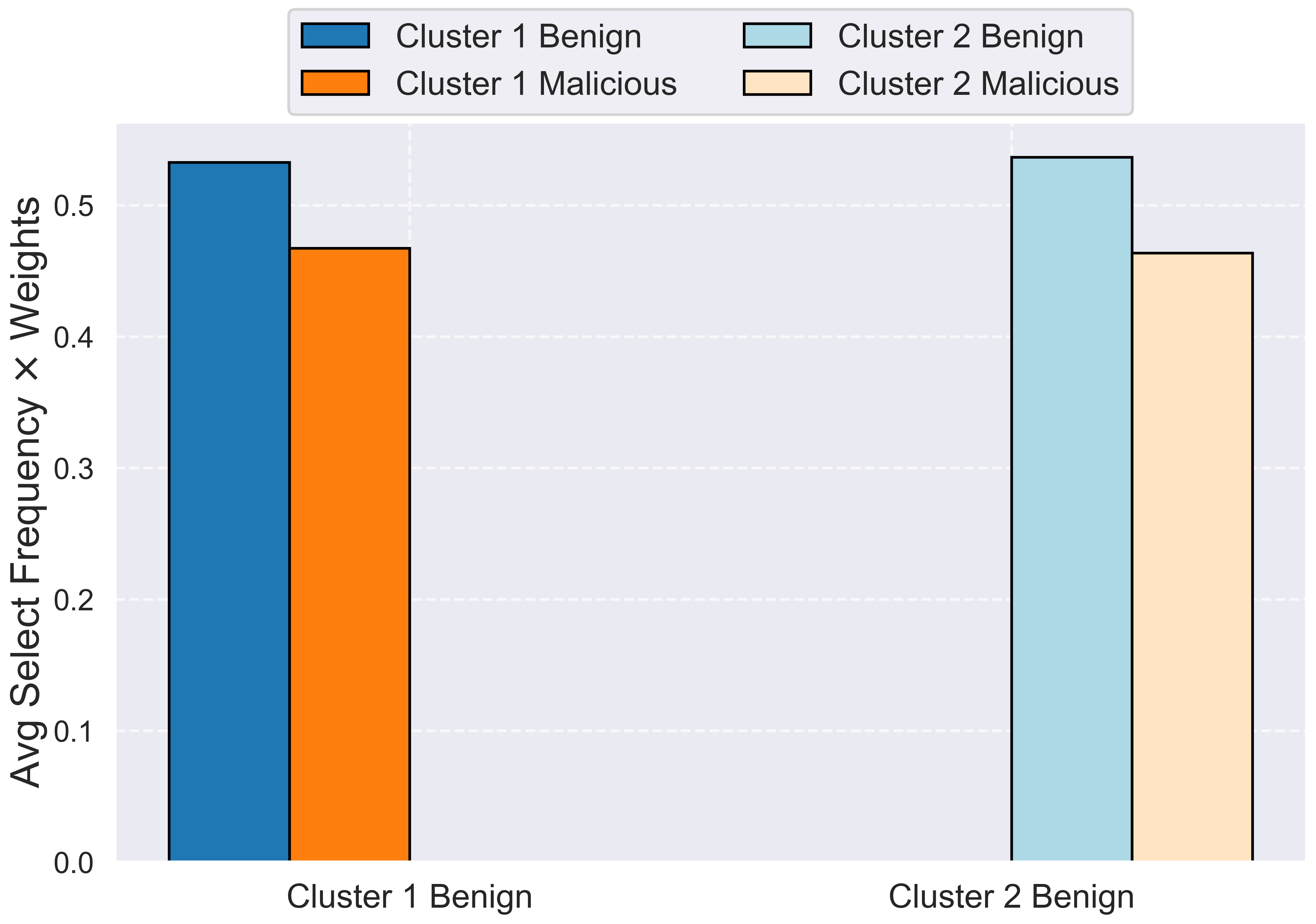}}\\
\vspace{1em}
\subcaptionbox{FedCB\textsuperscript{2}O Algorithm ($T_G=30$)\label{subfig:fedcb2o_avg_select_time}}{
		\includegraphics[ width=0.46\textwidth]{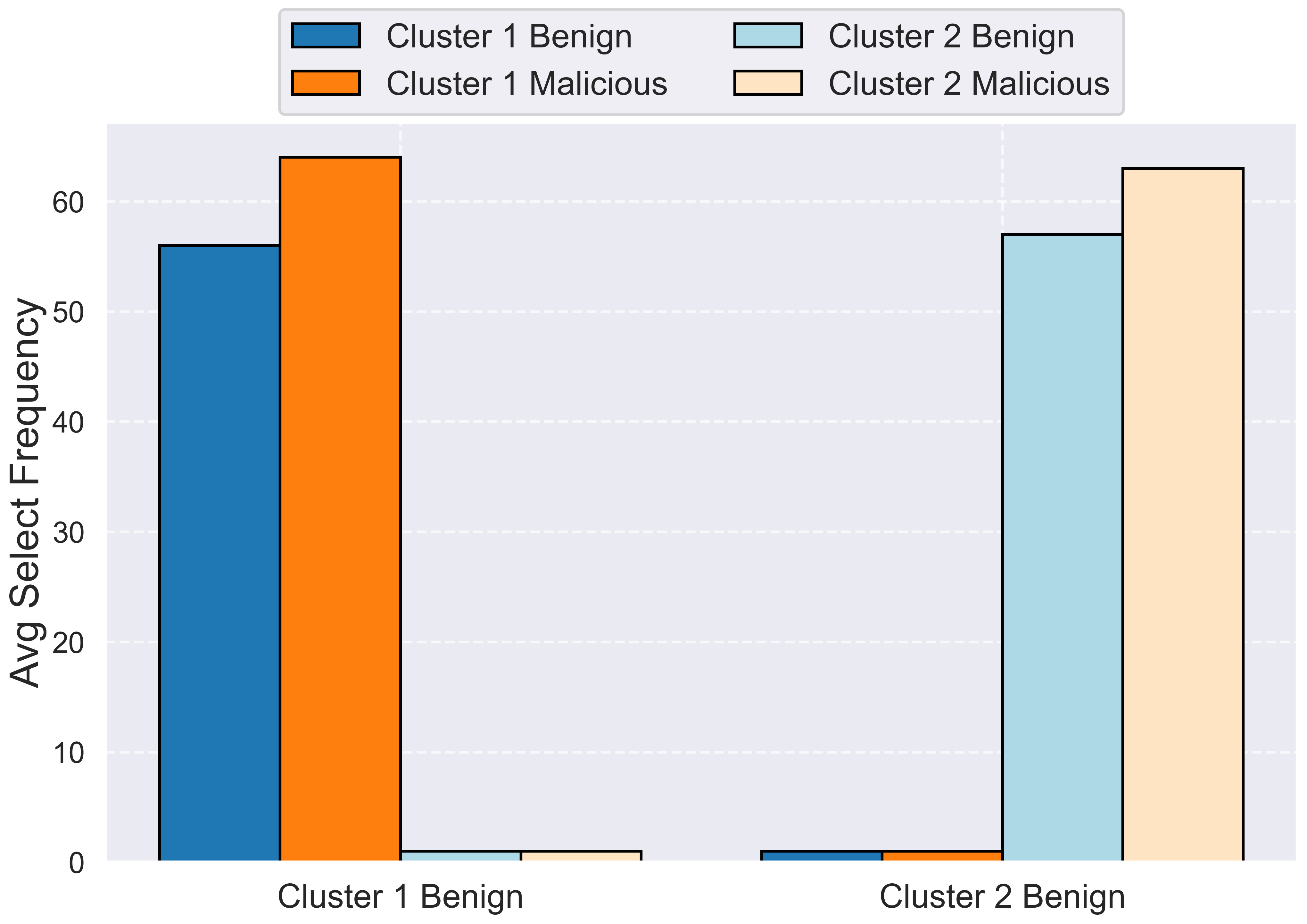}
	}
    \hspace{2em}
	\subcaptionbox{FedCB\textsuperscript{2}O Algorithm ($T_G=30$)\label{subfig:fedcb2o_avg_select_time_mul_weights}}{
\includegraphics[width=0.46\textwidth]{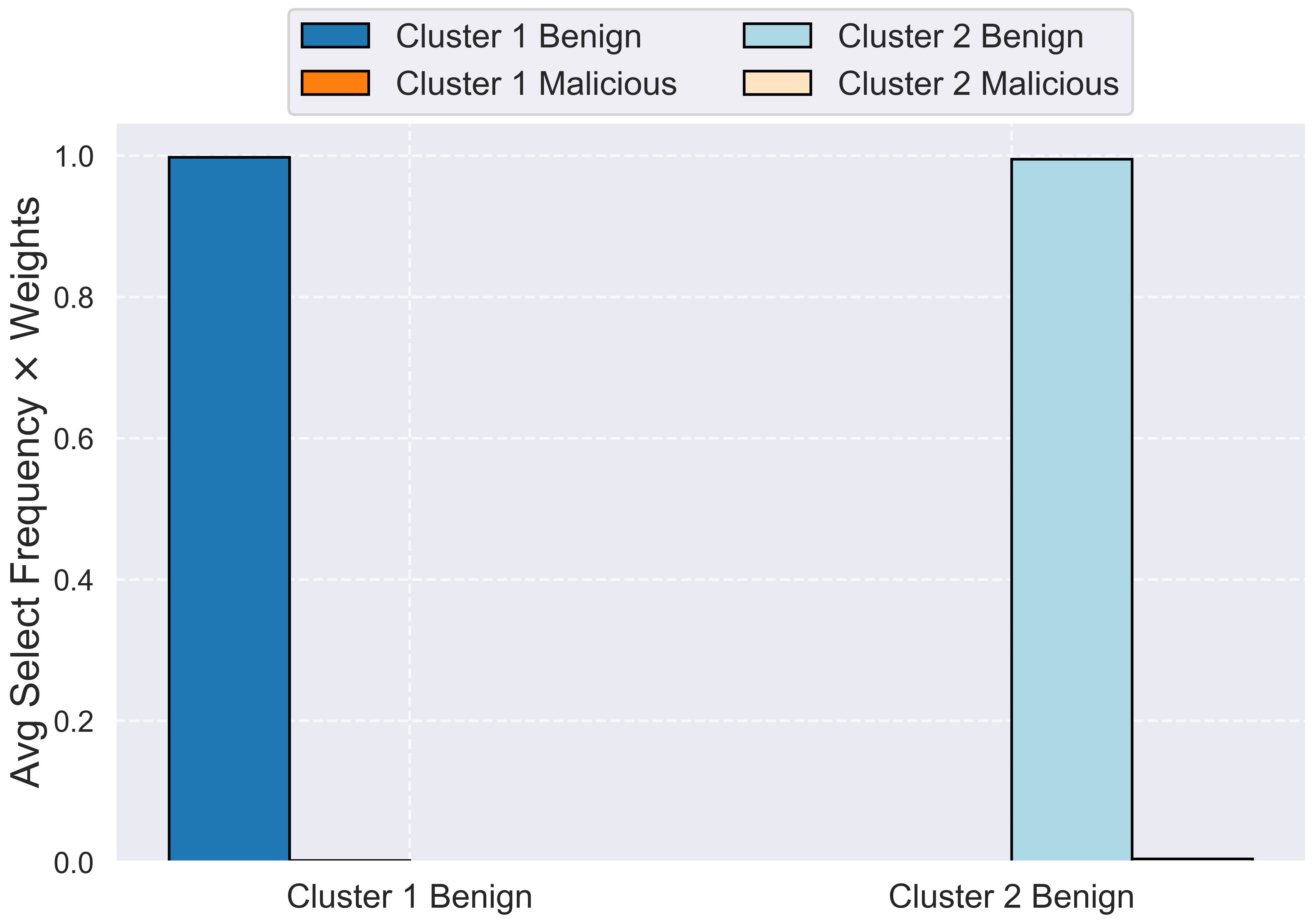}}
	\caption{{{\bf(a) and (c)}:} Average frequency at which benign agents select models from other benign or malicious agents within the same or a different cluster in the FedCBO/FedCB\textsuperscript{2}O algorithm. 
    For example, the dark blue (orange) bar with labeled ``Cluster 1 Benign'' (in x-axis) represents the average frequency of benign agents in cluster $1$ selecting models from other benign (malicious) agents in cluster $1$.
    The light blue (orange) bar labeled ``Cluster 1 Benign'' (in x-axis) corresponds to average frequency with which benign agents in cluster $1$ select models from other benign (malicious) agents in cluster $2$.
    Similar interpretations apply to the bars labeled ``Cluster 2 Benign'';
    {{\bf(b) and (d)}:} Normalized average selection frequency (as shown in Subfigures (a) and (c)) multiplied with the averaged weights assigned to other agents in the FedCBO/FedCB\textsuperscript{2}O algorithm.}
    \label{fig:avg_select_time_mul_weights}
\end{figure}

\begin{table}[!htb]
\centering
\caption{Comparison of FedCB\textsuperscript{2}O with different choices of the hyperparameter $T_G$ across three performance metrics (in $\%$) with standard deviations.}
\label{tab:FedCB2O_T_G}
\renewcommand\arraystretch{1.25}
    \begin{sc}
\begin{tabular}{ccccc}
\toprule
    \bf   & \bf \begin{tabular}{c}FedCB\textsuperscript{2}O \\ ($T_G=0$)\end{tabular} & \bf \begin{tabular}{c}FedCB\textsuperscript{2}O \\ ($T_G=20$)\end{tabular} & \bf \begin{tabular}{c}FedCB\textsuperscript{2}O \\ ($T_G=30$)\end{tabular} & \bf \begin{tabular}{c}FedCB\textsuperscript{2}O \\ ($T_G=40$)\end{tabular} \\
\midrule
Overall Acc & $82.79 \pm 0.14 $  & $83.68 \pm 0.07$ & $83.64 \pm 0.05$ & $\bf{83.76 \pm 0.16}$ \\
\midrule
Source Class Acc & $55.73 \pm 2.94 $ & $55.98 \pm 4.16$ & $\bf{57.64 \pm 1.83}$ & $56.87 \pm 2.49$ \\
\midrule
Attack Success Rate & $38.73 \pm 3.42 $ & $38.83 \pm 4.68$ & $\bf{37.31 \pm 2.33}$ & $38.07 \pm 2.96$ \\
\bottomrule
\end{tabular}
    \end{sc}
\end{table}

\begin{remark}[Importance of Robustness Criterion]\label{rem:importance_G}
Figures \ref{subfig:fedcb2o_avg_select_time} and \ref{subfig:fedcb2o_avg_select_time_mul_weights} underscore the importance of incorporating the robustness criterion $G_j$ into the weighted averaging process to defend against malicious agents.
Specifically, Figure \ref{subfig:fedcb2o_avg_select_time} shows that the agent selection mechanism in the FedCB\textsuperscript{2}O algorithm, implemented via the Probsampling method in Algorithm \ref{alg:localAggregation}, does not exclude malicious agents within the same cluster.
This behavior is expected since the agent sampling likelihood $P_n^j$ is updated only based on the evaluation of the lower-level objective $\AgentJObj$, as defined in \eqref{eq:update_sampling_likelihood}, which cannot differentiate between benign agents and malicious agents performing label-flipping attacks within the same cluster.
The underlying reason is that malicious agents possess more resources and data samples and are therefore able to train stronger models with smaller average losses compared to models trained by benign agents.
This even results in the frequency with which benign agents are select being smaller than the one with which malicious agents are selected, see Figure \ref{subfig:fedcb2o_avg_select_time}.
By incorporating the robustness criterion $G_j$ in the weighted average, however, minimal weights are assigned by benign agents to the selected malicious agents, as confirmed by Figure \ref{subfig:fedcb2o_avg_select_time_mul_weights}.
This effectively neutralizes the negative influence of malicious agents, empirically demonstrating why FedCB\textsuperscript{2}O algorithm can defend against LF attacks.

A comparison between Figures \ref{subfig:fedcbo_avg_select_time} and  \ref{subfig:fedcb2o_avg_select_time} reveals one more interesting yet reasonable observation.
Benign agents in the FedCB\textsuperscript{2}O algorithm are selected less frequently during the agent selection process compared to the FedCBO algorithm.
This phenomenon arises because FedCB\textsuperscript{2}O effectively eliminates the influence of malicious agents thanks to the robustness criterion.
However, in doing so, it also prevents benign agents from leveraging valuable information provided by malicious agents, resulting in a larger average loss $\AgentJObj$ for benign agents. 
Since the \textit{ProbSampling} agent selection mechanism relies solely on the lower-level objectives $\AgentJObj$, benign agents are consequently selected relatively less often.
\end{remark}

\section{Conclusions}\label{sec:conclusion}

In this paper, we abstracted the robust federated learning problem and formulated it as a bi-level optimization problem of the form \eqref{eq:bilevel_opt}.
This allowed us to establish a connection between the robust FL paradigm and consensus-based bi-level optimization (CB\textsuperscript{2}O),
a multi-particle metaheuristic optimization approach originally designed to solve nonconvex bi-level optimization problems.

On the theoretical side, we analyzed the CB\textsuperscript{2}O system in adversarial settings by taking a mean-field perspective.
We demonstrate the robustness of CB\textsuperscript{2}O against a wide range of attacks by proving its global convergence in mean-field law in the presence of malicious agents.
Additionally, we provide insights into how CB\textsuperscript{2}O defends against attacks by illustrating the role of key hyperparameters of the method.

On the algorithmic side, we extended CB\textsuperscript{2}O to the decentralized clustered federated learning setting and proposed FedCB\textsuperscript{2}O, a novel interacting particle system. 
To address practical demands and limitations present in FL applications, we designed an agent selection mechanism inspired by the consensus point computation in FedCB\textsuperscript{2}O.
This mechanism, which may be of independent interest for any FL algorithm involving an agent selection process,
is integrated into the FedCB\textsuperscript{2}O algorithm.
Compelling experiments in the DCFL setting confirm the effectiveness and robustness of the FedCB\textsuperscript{2}O algorithm despite the presence of malicious agents performing label-flipping attacks.

In future works, we aim to theoretically explore more practical settings where agents within the same cluster have similar but slightly distinct lower-level objective functions.
Furthermore, integrating the ``personalized'' robustness criterion, as proposed in Remark \ref{rem:func_G}, into the CB\textsuperscript{2}O framework presents a promising avenue for further research.

\vskip6pt

\section*{Acknowledgements}
All authors acknowledge the kind hospitality of the Institute for Computational and Experimental Research in Mathematics (ICERM) during the ICERM workshop ``Interacting Particle Systems: Analysis, Control, Learning and Computation''.\\
NGT was supported by the NSF grant DMS-2236447, and, together with SL would like to thank the IFDS at UW-Madison and NSF through TRIPODS grant 2023239 for their support. 
KR acknowledges the financial support from the Technical University of Munich and the Munich Center for Machine Learning, where most of this work was done.
His work there has been funded by the German Federal Ministry of Education and Research and the Bavarian State Ministry for Science and the Arts.
KR moreover acknowledges the financial support from the University of Oxford.
For the purpose of Open Access, the author has applied a CC BY public copyright licence to any Author Accepted Manuscript (AAM) version arising from this submission.
YZ was supported by the NSF grant DMS-2411396.

\bibliographystyle{abbrv}
\bibliography{biblio}



\end{document}

%% file: Macros.tex
\newcommand{\blue}{\color{blue}}
\newcommand{\red}{\color{red}}

\newcommand{\dummy}{\mathord{\color{gray}\bullet}}

\providecommand{\bbR}{\mathbb{R}}
\newcommand{\R}{\mathbb{R}}

\newcommand{\E}{\mathbb{E}}

\providecommand{\CN}{\mathcal{N}}
\providecommand{\CD}{\mathcal{D}}

\providecommand{\CC}{\mathcal{C}}

\providecommand{\CP}{\mathcal{P}}

\providecommand{\CV}{\mathcal{V}}

\providecommand{\CX}{\mathcal{X}}
\providecommand{\CY}{\mathcal{Y}}

\providecommand{\Id}{\mathrm{Id}}
\providecommand{\argmin}{\operatorname*{\arg\min}}
\providecommand{\arginf}{\operatorname*{\arg\inf}}
\newcommand{\supp}[1]{\operatorname{supp}(#1)} 
\providecommand{\dist}{\operatorname{dist}}

\newcommand{\thetaG}{\theta_{\mathrm{good}}^*}
\newcommand{\thetaGk}{\theta_{\mathrm{good}}^{*,k}}
\newcommand{\thetaCOne}{\theta_{t}^{1,i_1}}
\newcommand{\thetaCTwo}{\theta_{t}^{2,i_2}}
\newcommand{\thetaGtilde}{\tilde\theta_{\mathrm{good}}}
\newcommand{\mAlphaBeta}[1]{m^{G,L}_{\alpha, \beta}(#1)}
\newcommand{\mAlphaBetaLone}[1]{m^{G,L_1}_{\alpha, \beta}(#1)}
\newcommand{\mAlphaBetaLtwo}[1]{m^{G,L_2}_{\alpha, \beta}(#1)}

\newcommand{\mAlphaBetanoarg}{m^{G,L}_{\alpha, \beta}}

\newcommand{\Ibeta}[1]{I^L_{\beta}[#1]}
\newcommand{\qbeta}[1]{q^L_{\beta}[#1]}
\newcommand{\qbetahalf}[1]{q^L_{\beta/2}[#1]}
\newcommand{\qa}[1]{q^L_{a}[#1]}

\newcommand{\Qbeta}[1]{Q^L_{\beta}[#1]}
\newcommand{\QbetaLone}[1]{Q^{L_1}_{\beta}[#1]}
\newcommand{\QbetaLtwo}[1]{Q^{L_2}_{\beta}[#1]}

\newcommand{\Temp}{\kappa}
\newcommand{\AgentJObj}{\widetilde{L}_j}

\newtheorem{theorem}{Theorem}[section]
\newtheorem{proposition}[theorem]{Proposition}

\newtheorem{assumption}[theorem]{Assumption}

\theoremstyle{remark}

\newtheorem{remark}[theorem]{Remark}

\numberwithin{equation}{section}

\providecommand*{\absbig}[1]{\big|{#1}\big|} 
\providecommand*{\N}[1]{\left\|{#1}\right\|} 
\providecommand*{\Nnormal}[1]{\|{#1}\|} 
\providecommand*{\Nbig}[1]{\big\|{#1}\big\|} 

\newcommand{\overbar}[1]{\makebox[0pt]{$\phantom{#1}\mkern 1.5mu\overline{\mkern-1.5mu\phantom{#1}\mkern-1.5mu}\mkern 1.5mu$}#1}
\renewcommand{\underbar}[1]{\makebox[0pt]{$\phantom{#1}\mkern 1.5mu\underline{\mkern-1.5mu\phantom{#1}\mkern-1.5mu}\mkern 1.5mu$}#1}

\newcommand{\divergence}{\textrm{div}}

\newcommand{\omegaa}[0]{\omega_{\alpha}^G}

\newcommand{\indivmeasure}[0]{\varrho} 

\newif\ifrevisedOne
\newcommand{\revisedOne}[1]{%
	\ifrevisedOne
		\color{purple} #1 \color{black} 
	\else
		#1%
	\fi}
\revisedOnefalse
\newif\ifrevisedTwo

\revisedTwofalse